\documentclass[acmsmall]{acmart}

\setcopyright{none}
\renewcommand\footnotetextcopyrightpermission[1]{}
\usepackage{enumitem}
\usepackage[ruled,vlined]{algorithm2e}
\usepackage{thmtools,thm-restate}
\usepackage{graphicx}
\usepackage{epstopdf}
\usepackage{subcaption}
\usepackage{stfloats}

\AtBeginDocument{%
 \providecommand\BibTeX{{%
    \normalfont B\kern-0.5em{\scshape i\kern-0.25em b}\kern-0.8em\TeX}}}

\newtheorem{assumption}{Assumption}

\DeclareMathOperator*{\argmax}{arg\,max}
\usepackage{mathtools}
\DeclarePairedDelimiter{\ceil}{\lceil}{\rceil}

\begin{document}

\title{POND: Pessimistic-Optimistic oNline Dispatching}

\author{Xin Liu}
\affiliation{%
 \institution{University of Michigan, Ann Arbor}
 \streetaddress{1301 Beal Avenue}
 \city{Ann Arbor}
 \state{MI, 48109}
 \country{USA}}
\email{xinliuee@umich.edu}

\author{Bin Li}
\affiliation{%
 \institution{University of Rhode Island}
 \streetaddress{2 East Alumni Avenue}
 \city{Kingston}
 \state{RI, 02881}
 \country{USA}}
\email{binli@uri.edu}

\author{Pengyi Shi}
\affiliation{%
 \institution{Purdue University}
 \streetaddress{403 W State St Kran}
 \city{West Lafayette}
 \state{IN, 47907}
 \country{USA}}
\email{shi178@purdue.edu}

\author{Lei Ying}
\affiliation{%
 \institution{University of Michigan, Ann Arbor}
 \streetaddress{1301 Beal Avenue}
 \city{Ann Arbor}
 \state{MI, 48109}
 \country{USA}}
\email{leiying@umich.edu}

\renewcommand{\shortauthors}{Xin Liu, Bin Li, Pengyi Shi, and Lei Ying}

\begin{abstract}
 This paper considers constrained online dispatching with unknown arrival, reward, and constraint distributions.  We propose a novel online dispatching algorithm, named POND, standing for Pessimistic-Optimistic oNline Dispatching, which  achieves $O(\sqrt{T})$ regret and $O(1)$ constraint violation. Both bounds are sharp. Our experiments on synthetic and real datasets show that POND achieves low regret with minimal constraint violations. 
\end{abstract}



\maketitle
\thispagestyle{empty}
\fancyfoot{}

\section{Introduction}
Online dispatching refers to the process (or an algorithm) that dispatches incoming jobs to available servers in real-time. The problem arises in many different fields. Examples include routing customer calls to representatives in a call center, assigning patients to wards in a hospital, dispatching goods to different shipping companies, scheduling packets over multiple frequency channels in wireless communications, routing search queries to servers in a data center, selecting an advertisement to display to an Internet user, and allocating jobs to workers in crowdsourcing. 

In this paper, we consider the following discrete-time model over a finite horizon $T$ for the online dispatching problem. We assume there are $N$ types of jobs, the set of jobs is denoted by $\mathcal N = \{1,2 \cdots, N\},$ and $M$ types of servers, the set of servers is denoted by $\mathcal M = \{1,2 \cdots, M\}.$ Here a job may represent a patient who comes to an emergency room and needs to be hospitalized, an Internet user who browses a webpage, or a job submitted to a crowdsourcing platform; and a server may represent a hospital ward or a doctor in the emergency room, an advertisement of a product, or a worker registered at the crowdsourcing platform.  We assume that jobs of type $i$ arrive at each time slot $t$ according to a stochastic process $\Lambda_i(t)$ with unknown mean $\mathbb E[\Lambda_i(t)] = \lambda_i.$ The online dispatcher sends $x_{i,j}(t)$ of the $\Lambda_i(t)$ jobs to server $j,$ and receives reward 
$R_{i,j,s}(t)$ for the $s^{\text{th}}$ job, which is a random variable with mean $r_{ij},$ i.e., $\mathbb E[R_{i,j,s}(t)] = r_{i,j}, \forall s.$ 
Again, we assume $r_{i,j}$ is unknown to the dispatcher (e.g. click-through-rates are unknown to advertising platforms and average job completion quality is unknown to crowdsourcing platforms). The objective of the online dispatcher is to maximize the cumulative rewards over the $T$ time slots, i.e.,
\begin{equation}
    \sum_{t=0}^{T-1} \mathbb E\left[\sum_{i,j}\sum_{s=1}^{x_{i,j}(t)} R_{i,j,s}(t)\right] \label{obj-intro},
\end{equation} 
subject to cumulative constraints 
\begin{equation}
\sum_{t=0}^{T-1} \mathbb E\left[\sum_i w_{i,j}(t) x_{i,j}(t)\right] \leq \sum_{t=0}^{T-1} \mathbb E[\rho_j(t)], ~ \forall j \label{eq:cons}
\end{equation}
where the expectation is with respect to the randomness in job arrivals ($\Lambda$), rewards received ($R$), the constraint $\rho,$ and the dispatching policy. We consider a general set of linear constraints in this paper as in \eqref{eq:cons}. Assuming $w_{i,j}(t)=-1$ and $\rho_j(t)=\rho_j<0,$ $\forall 0\leq t < T,$ constraint \eqref{eq:cons} can represent a fairness constraint with $-\rho_j$ being a target workload level so that worker $j$ has a workload of at least $-\rho_j$ on average. This fairness constraint is much desired in many systems because an unfair load distribution across servers (such as customer representatives) often leads to the loss of trust of the system and the loss of work efficiency. Constraint \eqref{eq:cons}  can also be interpreted as a budget constraint if $w_{i,j}(t)$ is the cost incurred for server $j$ to complete a type-$i$ job at time slot $t$, and $\rho_j(t)$ is new budget allocated to server $j$ at time slot $t.$ In this paper, we assume $\rho_j(t)$  is unknown apriori and revealed at the end of time slot $t$. The information of $w_{i,j}$ and $\rho_j$ are unknown to the dispatcher. 

The focus of this paper is on efficient online dispatching algorithms to maximize the cumulative reward \eqref{obj-intro} under constraints in the form of \eqref{eq:cons}. We note that different versions of this problem have been studied in different fields. For example, without constraints and assuming job types are not related, the problem is a contextual multi-armed bandit problem where each job is a context (called ``One Bandit per Context" in \cite{LatSze_20}). With a special form of the fairness constraint, the problem is called fair contextual multi-armed bandits \cite{CheAleLuo_20}. Different from existing work, this paper considers general constraints and establishes {\em sharp} regret and constraint violation bounds with unknown job arrival distributions, reward distributions, and constraint distributions. A detailed review of related work can be found in Section \ref{sec:related}. We next summarize the main contributions of this paper. 

\begin{itemize}
\item {\bf Algorithm.} We propose a new online dispatching algorithm, called pessimistic-optimistic online dispatching or \textbf{POND} in short, which combines the celebrated Upper Confidence Bound (UCB) \cite{AurCesFis_02}, an {\em optimistic}  approach for estimating the rewards, and the celebrated MaxWeight based on virtual queues, where the virtual queues are updated with reduced service rates, so {\em pessimistically} tracking the constraint violations.  POND includes three key components:
\begin{enumerate}
    \item \textbf{UCB} -- POND utilizes UCB or a UCB-type algorithm (e.g. MOSS, Minimax Optimal Strategy \cite{AudBub_09}) to learn the mean rewards $r_{ij}$; 
    \item \textbf{Virtual Queues} -- Virtual queues track the level of constraint violation so far. A  $O(1/\sqrt{T})$  ``tightness'' is added in the virtual queue updates so that the virtual queues overestimate the constraint violations. 
    \item \textbf{MaxWeight} -- At each time slot, the incoming jobs are allocated to servers to maximize the total ``weight'', where the weight of allocating a type-$i$ job to the $j$th server is a linear combination of the estimated reward $\hat{r}_{ij}$ and the values of the virtual queues, to balance between maximizing  rewards and avoiding constraint violations. 
\end{enumerate}

\item {\bf Theory.}  We prove that over $T$ time slots, POND achieves $O(\sqrt{T})$ regret with $O(1)$ constraint violations.  These bounds are sharp because our regret bound $\Omega(\sqrt{T})$  matches the (reward)-distribution-independent lower bound for multi-armed bandit problems without constraints \cite{BubCes_12}; and otherwise, $O(1)$ constraint violation is the smallest possible. Our main proof combines the Lyapunov drift analysis \cite{Nee_10,SriYin_14} for queueing systems and the regret analysis for multi-armed bandit problems, bridged by the ``tightness'' introduced in virtual queues. In particular, both the regret and the constraint violations depend on the level of ``tightness'' added to the virtual queues as a large tightness reduces the constraint violations but leads to more suboptimal dispatches and vice versa.  By optimizing the level of ``tightness'', POND achieves both optimal regret and constraint violations. 

\item {\bf Experiments.} We verified that POND balances the regret and constraint violations effectively using experiments based on both synthetic data and a real dataset on online tutoring. Specifically, Figure \ref{intro:regret and cv v.s. time} shows in the experiments with synthetic data, POND achieves $O(\sqrt{T})$ regret and $O(1)$ constraint violations by adding ``tightness'', and achieves $O(\sqrt{T})$ regret and $O(\sqrt{T})$ constraint violations without the ``tightness''. Moreover, POND significantly outperforms Explore-then-Commit algorithm (ETC)  in the experiments. For example, POND with tightness $0.5$ (marked with circles) over $T=10,000$ has a regret of $323$ with capacity violation of 7 and resource violation of -35, while ETC has a regret of $536$ (\~70\% higher) with capacity violation of $-48$ and resource violation of $250$ (details can be found in Section \ref{sec: experiments}).
\begin{figure}[h]
\centering
\begin{subfigure}{.33\textwidth}
  \centering
  \includegraphics[width=1.0\linewidth]{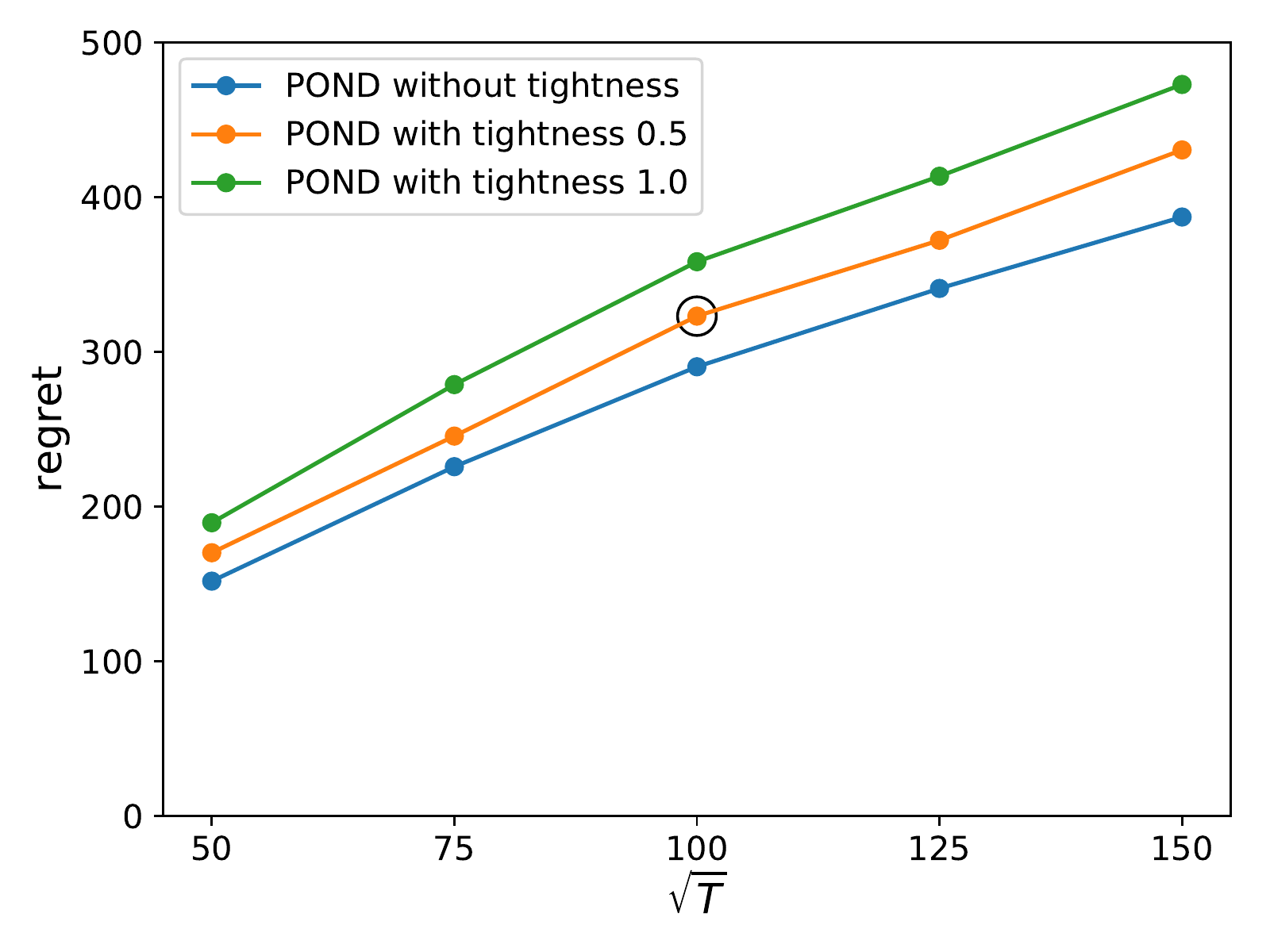}
  \caption{Regret}
\end{subfigure}%
\begin{subfigure}{.33\textwidth}
  \centering
  \includegraphics[width=1.0\linewidth]{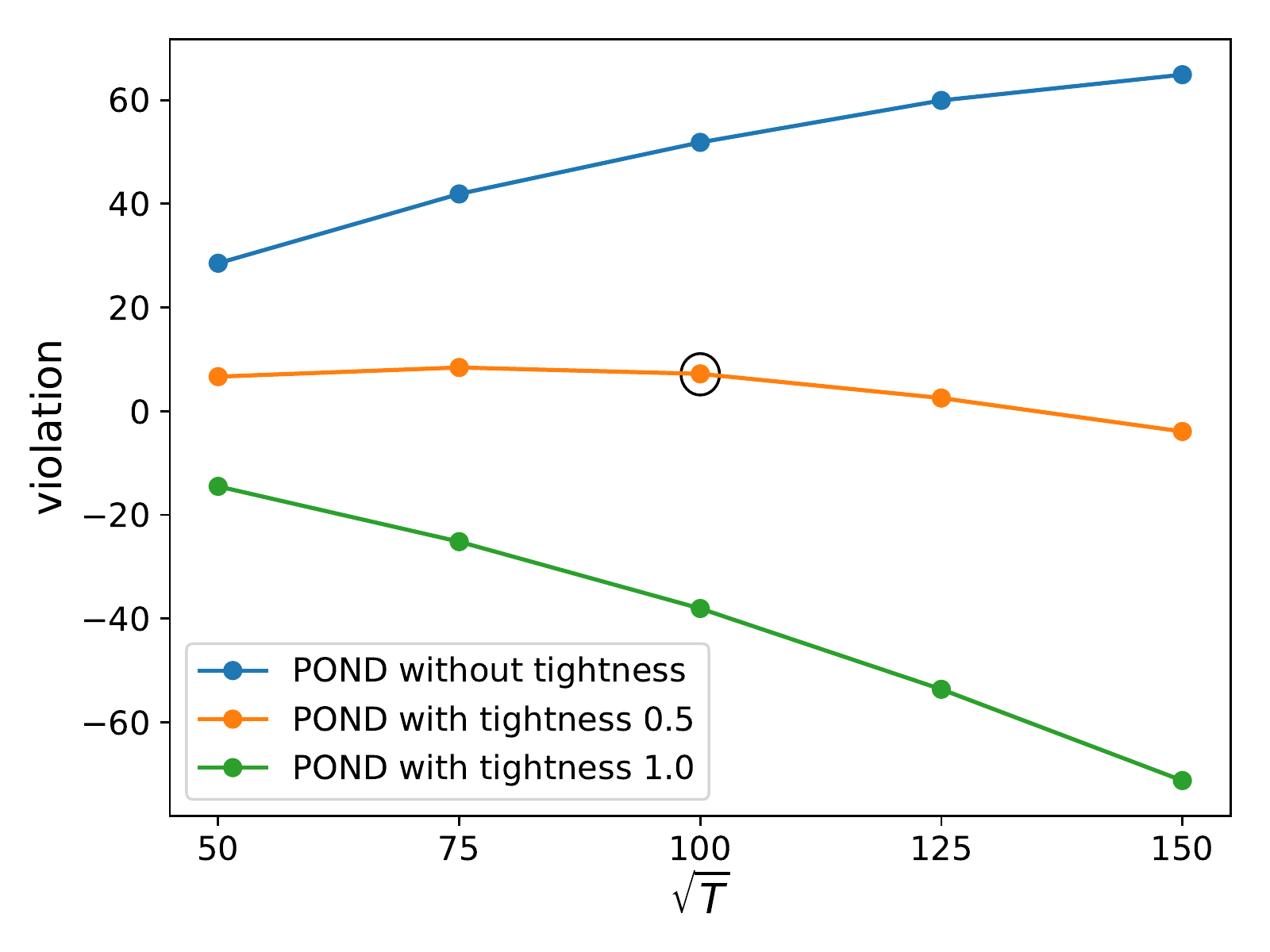}
  \caption{Capacity violation}
\end{subfigure}
\begin{subfigure}{.33\textwidth}
  \centering
  \includegraphics[width=1.0\linewidth]{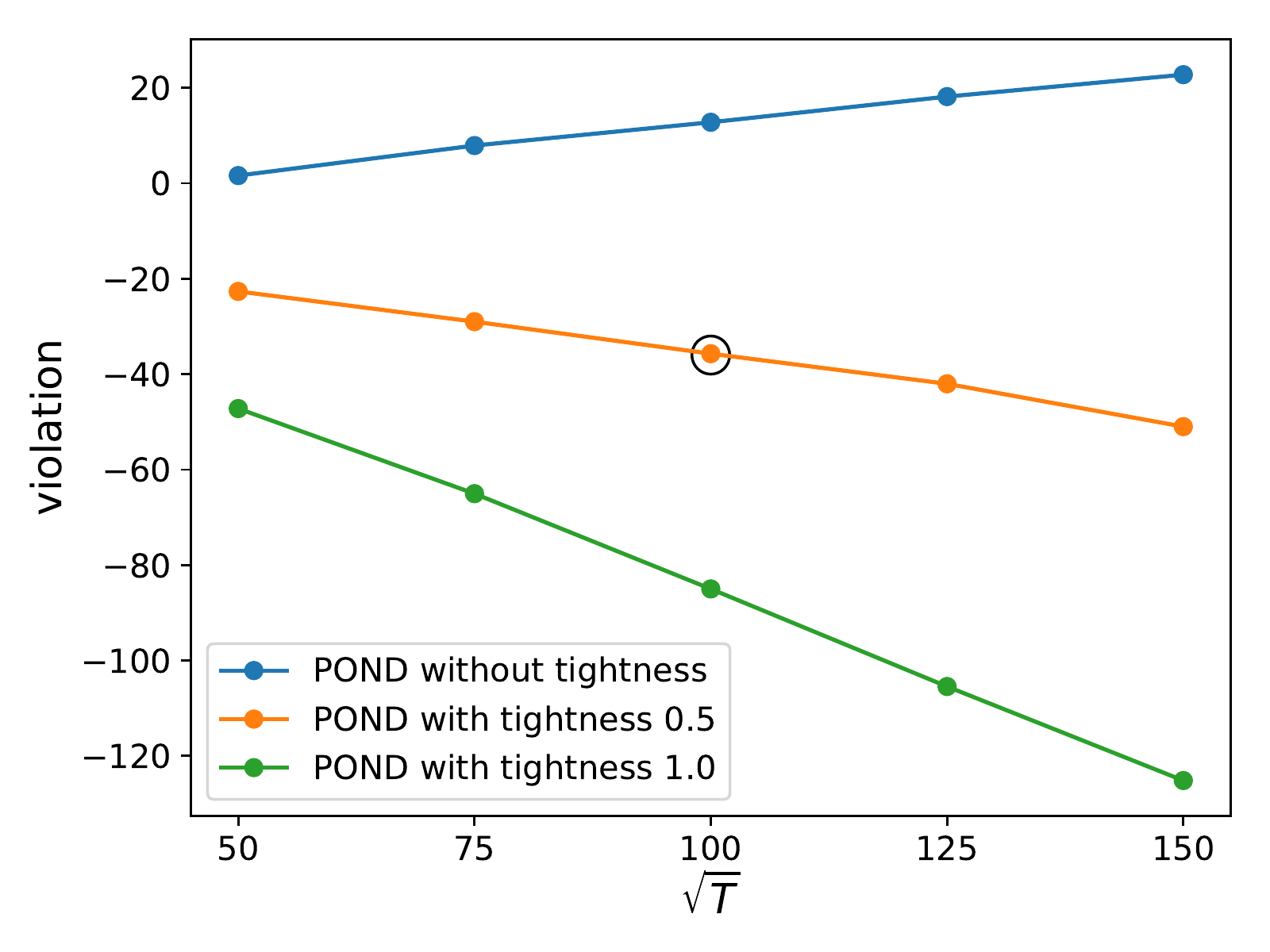}
  \caption{Resource violation}
\end{subfigure}
\caption{Regret and constraint violation versus $\sqrt{T}.$}
\label{intro:regret and cv v.s. time}
\end{figure} 
\end{itemize}

\subsection{Related work}
\label{sec:related}
Online dispatching has been widely studied in different fields. In particular, the problem is related to multi-armed bandits, online convex optimization, and load-balancing/scheduling in queueing systems. We next review related work by organizing them into these three categories.   

\begin{itemize}
\item {\bf Multi-Armed Bandits.} Online dispatching is particularly challenging when reward distributions are unknown apriori. 
In such a case, the problem has often been formulated as a multi-armed bandit (MAB) problem, see e.g. \cite{JohKamKan_16, HsuXuLin_18,  LiLiuJi_19, FerSimWan_18}. 
\cite{JohKamKan_16} considered online matching between jobs and servers with unknown server types and proposed an exploration (for identifying server types) and exploitation (for maximizing profits) algorithm, which achieves a steady-state regret of  $O(\log M/M),$  where $M$ is the number of servers.  \cite{HsuXuLin_18} studied an online task assignment with random payoffs subject to capacity constraints. A joint UCB learning and dynamic task allocation algorithm is proposed with $O(T)$ regret.
\cite{CayErySri_20} considered budget-constrained bandits ($B$ is total budget across all arms) with correlated cost and reward distributions and proposed an enhanced-UCB algorithm to achieve $O(\log (B))$ regret by exploiting the correlation. This paper considers general constraints (e.g. fairness constraint) significantly beyond the total budget constraint in \cite{CayErySri_20}.
Recently, \cite{CheAleLuo_20} considered an adversarial contextual bandit with fairness constraints. The proposed algorithm achieves  $O(\sqrt{T})$ regret with strict fairness guarantees when the context distribution is known (i.e., arrival distribution is known). For unknown context distribution, an algorithm has been proposed to achieve $O(\sqrt{T})$ regret and $O(\sqrt{T})$ constraint violation. \cite{LiLiuJi_19} studied a combinatorial sleeping bandits problem under fairness
constraints and proposed an algorithm based on UCB that achieves $O(\sqrt{T})$ regret and $O(\sqrt{T})$ constraint violation with a properly chosen tuning parameter. The model studied in \cite{LiLiuJi_19} is similar to ours.
This paper proposes a new algorithm --- POND, which achieves $O(\sqrt{T})$ regret and $O(1)$ constraint violation, and both are sharp. 
\cite{FerSimWan_18} also studied a similar problem with hard budget constraint, where the amount of budgets are known apriori. They proposed an algorithm based on the Thompson sampling and linear programming, which achieves $O(\sqrt{T})$ Bayesian regret. Their algorithm cannot be applied to our model because solving the linear programming requires the constraint parameters ($\rho$) to be known apriori. 

\item {\bf Online Convex Optimization.} Another line of research that is related to ours is online convex optimization \cite{Zin_03, Haz_13} for online resource management. For example, \cite{MahJinYan_12} considered online convex optimization under static constraints and proposed an online primal-dual algorithm that guarantees $O\left(T^{\max\{\beta,1-\beta\}}\right)$ regret and $O\left(T^{1-\frac{\beta}{2}}\right)$ constraint violation, where $\beta \in [0, 1]$ is a tuning parameter. The result has recently been improved in \cite{YuNee_20} to achieve $O(\sqrt{T})$ regret and $O(1)$ constraint violations under the assumption that the algorithm has the access to the full gradient at each step, which does not hold in our model.  Online convex optimization with stochastic constraints (i.i.d. assumption) has also been studied in \cite{YuNeeWei_17} which proposes an online primal-dual with proximal regularized algorithm, to achieve $O(\sqrt{T})$ regret and $O(\sqrt{T})$ constraint violations. They later relaxed the Slater's condition in \cite{YuNeeWei_17} and obtained $O(\sqrt{T})$ regret and $O(\sqrt{T})$ constraint violations 
in \cite{WeiYuNee_20} .
Assuming the objective function and constraints are revealed exactly before making a decision, two recent papers \cite{LuBalMir_20} and \cite{BalLuMir_20} established $O(\sqrt{T})$ regret. In our model, the realizations of the instantaneous rewards and  constraints are known after the dispatching decisions are made.

\item {\bf Load-Balancing and Scheduling.} There are recent work on load balancing and scheduling in queueing systems with
bandit learning.
\cite{KriAraJoh_18} studied scheduling in multi-class queueing systems with single and parallel servers with unknown arrival and service distributions and showed that the learned $c\mu$-rule  can achieve constant regret. \cite{KriSenJoh_16,KriSenJoh_21}  studied ``Queueing Bandits'', which is a variant of the classic multi-armed bandit problem in a discrete-time queueing system with unknown service rates. They showed that the proposed algorithm Q-ThS achieves $O(\log^3T)$ regret in the early stage and $O(\log T/T)$ in the late stage. 
\cite{ChoJosWan_21} further studied ``Queueing Bandits'' with unknown service rates and queue lengths. They focused on a class of weighted random routing policies and showed an $\epsilon_t$-exploration policy can achieve $O(\sqrt{T}\log T)$ regret.
These papers \cite{KriAraJoh_18,KriSenJoh_16,KriSenJoh_21,ChoJosWan_21}  considered queue-regret, which is the difference of queue lengths between the proposed algorithm and the optimal algorithm, and is fundamentally  different from the reward regret considered in this paper. \cite{TarSenVec_19} studied online channel-states partition and user rates allocation with the bandit approach and proposed epoch-greedy bandit algorithm, which achieves $O(T^{2/3}\log T)$ regret.
\end{itemize}

\section{System Model, POND, and Main Results}
\label{sec:system-model}

We consider an online dispatching system with a set of job types $\mathcal N = \{1,2 \cdots, N\}$ and a set of servers $\mathcal M = \{1,2 \cdots, M\}.$ Figure \ref{fig:system model} illustrates an example with $N=2$ and $M=3$.
We assume a discrete-time system with a finite time horizon $T$ with time slot $t \in [0, 1, \cdots, T-1].$ Jobs arrive according to random processes. 
The number of type-$i$  arrivals at time slot $t$ is denoted by $\Lambda_i(t),$ which is a random variable with $\mathbb E[\Lambda_i(t)] = \lambda_i.$ 
We define  $x_{i,j}(t)$ to be the decision variable that is the number of type-$i$ jobs assigned to server $j$ at time slot $t.$ We further define $\mu_j(t)$ to be the number of jobs server $j$ can complete at time slot $t,$ assuming there are a sufficient number of jobs so the server does not idle. We assume $\mu_j(t)$ is a random variable with $\mathbb E[\mu_j(t)] = \mu_j.$ 
When a type-$i$ job is assigned to server $j,$ we receive reward $R_{i,j}(t)$ immediately\footnote{The same results hold for a model where the reward is received after server $j$ completes a type-$i$ job.}.  We model $R_{i,j}(t)$ to be a random variable in with an unknown distribution with $\mathbb E[R_{i,j}(t)]=r_{i,j},$ which is the case in many applications such as order dispatching in logistics, online advertising, and patient assignment in healthcare. Furthermore, we assume the arrival processes $\{\Lambda_i(t)\},$  service processes $\{\mu_j(t)\},$
and the reward processes $\{R_{i,j}(t)\}$ are i.i.d across job types, servers and time slots. 

\begin{figure}[h]
\centering
  \includegraphics[width=2.3in]{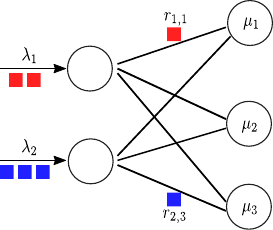}
  \caption{An example of our system with two types of jobs and three servers.}
  \label{fig:system model}
\end{figure}

The goal of this paper is to develop online dispatching algorithms to maximize the cumulative reward over $T$ time slots: 
\begin{align}
\sum_{t=0}^{T-1}\mathbb E\left[\sum_{i,j}\sum_{s=1}^{x_{i,j}(t)}R_{i,j,s}(t)\right] \label{obj:org}    
\end{align} 
where $R_{i,j,s}(t)$ are i.i.d. random variables across $s$ and has the same distribution with $R_{i,j}(t).$ The objective function in \eqref{obj:org} is equivalent to
\begin{align}
\sum_{t=0}^{T-1}\mathbb E\left[\sum_{i,j}r_{i,j}x_{i,j}(t)\right] \label{obj}
\end{align}
because $x_{i,j}(t)$ is independent of $R_{i,j,s}(t), \forall s.$ For a resource-constrained server system, we aim to maximize the objective \eqref{obj} subject to a set of constraints, including the capacity, fairness and resource budget constraints (all these constraints are unified into general forms), formulated as follows: 

\textit{Optimization Formulation:}
\begin{align}
    \max_{\mathbf x(t)} & ~  \sum_{t=0}^{T-1}\mathbb E\left[\sum_{i,j}r_{i,j}x_{i,j}(t)\right] \label{obj-T}\\
    \text{s.t.}
    &~ \sum_{j} x_{i,j}(t)= \Lambda_i(t), \forall i \in \mathcal N, ~x_{i,j}(t) \geq 0, \forall i \in \mathcal N, j \in \mathcal M, \label{arrival-T}\\
    & ~ \sum_{t=0}^{T-1}  \mathbb E\left[\sum_{i} w_{i,j}^{(k)}(t)x_{i,j}(t)\right] \leq \sum_{t=0}^{T-1}  \mathbb E\left[\rho^{(k)}_j(t)\right], ~ \forall j \in \mathcal M, \forall k \in \mathcal K. \label{resource limit-T}
\end{align}
where $x_{i,j}(t)$ is the number of type-$i$ jobs assigned to server $j$ at time slot $t$ and $\mathbf x(t)$ is its matrix version in which the $(i,j)$th entry is $x_{i,j}(t);$ \eqref{arrival-T} represents the allocating conservation for job arrivals; and \eqref{resource limit-T} can represent the capacity, fairness and resource budget constraints, where
$w_{i,j}^{(k)}(t)$ is the ``weight'' of a type-$i$ job to server $j$ and $\rho_j^{(k)}(t)$ is the corresponding ``requirement''. 
We next list a few examples of $w_{i,j}^{(k)}(t)$ and $\rho_j^{(k)}(t)$ so that the constraint represents the  capacity, fairness and resource budget constraint, respectively:
\begin{itemize}
\item Let $w_{i,j}^{(k)}(t) = 1$ and $\rho_{j}^{(k)}(t) = \mu_j(t), \forall t.$ The constraint represents the \emph{average} capacity constraints for server $j;$ 
\item Let $w_{i,j}^{(k)}(t) = -1$ and $\rho_{j}^{(k)}(t) = -\sum_{i}\Lambda_{i}(t)d_j, \forall t.$ The constraint represents  the \emph{average} fairness constraints, that is, each server $j$ needs to serve at least $d_j$ fraction of the total arrivals;
\item Let $w_{i,j}^{(k)}(t) $ be the amount of resource consumed by a type-$i$ job at server $j,$ and $\rho_{j}^{(k)}(t)$ be the budget added to server $j.$ The constraint represents the \emph{average} resource budget constraints. 
\end{itemize}
In this paper, we assume for given $k,$ $w_{i,j}^{(k)}(t)\geq 0$ for all $i,$ $j,$ and $t;$ or $w_{i,j}^{(k)}(t)\leq 0$ for all $i,$ $j,$ and $t;$ $w_{i,j}^{(k)}(t)$ and $\rho_{j}^{(k)}(t)$ are i.i.d across $i,$ $j,$ and $t.$

There are two major challenges in solving \eqref{obj-T}-\eqref{resource limit-T} in real time: unknown reward distributions, and unknown statistics of arrival processes, service processes and constraint parameters. To tackle unknown reward distributions, we utilize UCB learning~(e.g. UCB \cite{AurCesFis_02} or MOSS \cite{AudBub_09}), to learn (estimate) $r_{i,j}.$ 
To deal with unknown arrival  processes, service processes and stochastic constraints, we maintain virtual queues on the server side.
The virtual queues are related to dual variables \cite{Nee_10, SriYin_14}, which are used to track the constraint violations. 

\textit{Virtual Queues:}
\begin{align} 
Q_j^{(k)}(t+1) =& \left[Q_j^{(k)}(t) + \sum_{i}w_{i,j}^{(k)}(t) x_{i,j}(t) - \rho_j^{(k)}(t) + \epsilon \right]^{+}, \forall j \in \mathcal M, \forall k \in \mathcal K. \label{virtual queue}
\end{align}
The operator $(x)^{+} = \max(x, 0).$ $Q_j^{(k)}(t)$ is the virtual queue associated to the $k^{\text{th}}$ constraint imposed on server $j.$ $\sum_{i}w_{i,j}^{(k)}(t) x_{i,j}(t)$ is the ``total weight'' (e.g. capacity or budget consumption) on server $j$ and $\rho_{j}^{(k)}(t)$ is the ``requirement'' (e.g. capacity or budget limit) on the server $j.$ $\epsilon$ is a tightness constant that decides the trade-off between the regret and constraint violations, which we will specify in the proof later. This idea of adding tightness was inspired by the adaptive virtual queue (AVQ) used for the Internet congestion control \cite{KunSri_01}. We will see that by choosing  $\epsilon=O(1/\sqrt{T}),$ the algorithm presented next can achieve $O(\sqrt{T})$ regret and $O(1)$ constraint violations.

\subsection{POND}\label{sec: pond}
To maximize the cumulative reward in \eqref{obj} while keeping constraint violations reasonably small, we incorporate the learned reward and virtual queues in \eqref{virtual queue} to design POND - Pessimistic-Optimistic oNline Dispatching (Algorithm \ref{main alg}).

In Algorithm \ref{main alg}, we first utilize the classic UCB algorithm or MOSS algorithm to learn the reward $\hat r_{i,j}(t),$ then allocate the incoming jobs according to a ``max-weight'' algorithm, and finally update virtual queues and reward estimation according to the max-weight dispatching decisions. Note that $\hat{r}_{i,j}(t)=\infty$ when $N_{i,j}(t-1)=0,$ which implies that $\eta_{i,j}(t)=\infty$. When multiple $\eta_{i,j}(t)=\infty,$ we break the tie uniformly at random. 
In weight $\eta_{i,j}(t) = V \hat r_{i,j}(t) - \sum_k w_{i,j}^{(k)}(t)Q_j^{(k)}(t),$ parameter $V$ is chose to be $O(\sqrt{T})$ to balance the reward and virtual queues (constraint violations). When the virtual queue $Q_j^{(k)}(t)$ associated to capacity constraint is large (capacity constraint of server $j$ is violated too often), which implies the algorithm allocates too many jobs to server $j,$ weight $\eta_{i,j}(t)$ tends to be small so POND is less likely to allocate new incoming jobs to server $j$. Similarly, when virtual queue $Q_j^{(k)}(t)$ associated to fairness constraint is large (fairness constraint of server $j$ has been violated), which implies server $j$ has not received sufficient number of jobs,  weight $\eta_{i,j}(t)$ tends to be large (recall $w_{i,j}^{(k)}(t)=-1$ in fairness constraints) so POND  is more likely to allocate new incoming jobs to server $j.$

\begin{algorithm}[H]
\SetAlgoLined
 Input: $V, \epsilon,$ $Q_j^{(k)}(0)=0, \forall k,$ and $\bar r_{i,j}(-1)=N_{i,j}(-1)=0, \forall i, j.$\\
\For{$t=1,\cdots, T-1$}
{
UCB learning: $\hat r_{i,j}(t) = \bar r_{i,j}(t-1) + \sqrt{\frac{\log T}{N_{i,j}(t-1)}};$ \\
or MOSS learning: $\hat r_{i,j}(t) = \bar r_{i,j}(t-1) + \sqrt{\frac{2}{N_{i,j}(t-1)}\log \frac{T}{M\cdot N_{i,j}(t-1)}}.$\\ 

Compute the weight of a type-$i$ job to server $j$ at time $t:$
\begin{align*}
\eta_{i,j}(t) = V \hat r_{i,j}(t) - \sum_k w_{i,j}^{(k)}(t)Q_j^{(k)}(t), \forall j.
\end{align*}

Observe jobs arrival $\Lambda_i(t)$ and do max-weight allocation:
 \begin{align*}
     x_{i,j}(t) \in \argmax \limits_{\Lambda_i(t)=\sum_j x_{i,j} }~ \sum_{i, j}
\eta_{i,j}(t) x_{i,j}.
 \end{align*}\\

Update virtual queues:
\begin{align*}
  Q_j^{(k)}(t+1) =& \left[Q_j^{(k)}(t) + \sum_{i}w_{i,j}^{(k)}(t) x_{i,j}(t) - \rho_j^{(k)}(t) + \epsilon \right]^{+}, \forall j,k.
 \end{align*} 
 
Update the estimation of $\bar r_{i,j}(t)$ according to the rewards received:
 \begin{align*}
 N_{i,j}(t) &= N_{i,j}(t-1)+ x_{i,j}(t), \\
 \bar r_{i,j}(t) &= \frac{\bar r_{i,j}(t-1) N_{i,j}(t-1) + \sum_{s=1}^{x_{i,j}(t)} R_{i,j,s}(t)}{N_{i,j}(t)}.   
 \end{align*}
 }
\caption{POND Algorithm.}
\label{main alg}
\end{algorithm}

We remark that MOSS learning achieves the tight regret bound $O(\sqrt{T}),$ and UCB achieves regret bound $O(\sqrt{T\log T}).$ However, in practice, MOSS learning might explore too much and suffer from suboptimality and instability (instability means the distribution of regret under MOSS might not be well-behaved, for example, the variance could be  $O(T)$) \cite{LatSze_20}. 

\subsection{Main Results}
To analyze the performance of POND, we compare it with an offline optimization problem given the reward, arrival, service and constraint parameters. By abuse of notation, define $x_{i,j} = \frac{1}{T}\sum_{t=0}^{T-1}\mathbb E[x_{i,j}(t)],$ $w_{i,j}^{(k)} = \mathbb E[w_{i,j}^{(k)}(t)]$ and $\rho_j^{(k)} = \mathbb E[\rho_j^{(k)}(t)]$ in the optimization problem \eqref{obj-T}-\eqref{resource limit-T}. We consider the following offline optimization problem (or fluid optimization problem): 
\begin{align}
    \max_{\mathbf x} & ~  \sum_{i,j} r_{i,j} x_{i,j} \label{obj-fluid}\\
    \text{s.t.} & ~ \lambda_i = \sum_{j}  x_{i,j}, ~\forall i \in \mathcal N,  ~x_{i,j} \geq 0, \forall i \in \mathcal N, j \in \mathcal M, \label{arrival-fluid} \\
    & ~ \sum_{i}  w_{i,j}^{(k)} x_{i,j} \leq \rho_j^{(k)}, ~\forall j \in \mathcal M, \forall k \in \mathcal K. \label{resource limit-fluid} 
\end{align}
where $x_{i,j}$ corresponds to the average number of type-$i$ jobs assigned to server $j$ per time slot;
\eqref{arrival-fluid} includes throughput constraints; 
\eqref{resource limit-fluid} includes capacity constraints, fairness constraints and resource budget constraints. 

Next, we define performance metrics, including regret and constraint violation and present an informal version of main theorem.

\textit{Regret:}
Let $\mathcal X$ be the feasible set and $\mathbf x^*$ be the solution to the offline problem \eqref{obj-fluid}-\eqref{resource limit-fluid}. We define the regret of an online dispatching algorithm to be 
\begin{align*}
\mathcal R(T) = T \sum_{i,j}r_{i,j} x_{i,j}^* - \sum_{t=0}^{T-1}\mathbb E\left[\sum_{i,j}r_{i,j}x_{i,j}(t)\right].
\end{align*}

\textit{Constraint violation:}
We define constraint violations to be
\begin{align}
    \mathcal V(T) = \sum_{j}\sum_{k} \left(\sum_{t=0}^{T-1}  \mathbb E\left[\sum_{i} w_{i,j}^{(k)}(t)x_{i,j}(t)-\rho_j^{(k)}(t)\right] \right)^+,  \nonumber
\end{align}
which includes violations from capacity, fairness, and budget constraints. Note that $\mathcal V(T) \leq C$ implies that each constraint violation is bounded by $C.$

\begin{theorem}[Informal Statement]
Assuming bounded arrivals and rewards and let $V=O(\sqrt{T})$ and $\epsilon=O(1/\sqrt{T}),$ 
the regret and constraint violations under \text{POND} are
\begin{align*}
\mathcal R(T) = O(\sqrt{T}) ~\text{and}~ \mathcal V(T) = O(1).
\end{align*}
\label{them:informal}
\end{theorem}
The formal statement of the theorem can be found in Theorem \ref{thm:formal-UCB} and Theorem \ref{thm:formal-MOSS}. 

In Theorem \ref{them:informal}, the time horizon $T$ is known to POND in advance. We remark when the time horizon $T$ is unknown, the doubling trick \cite{LatSze_20,BesKau_17} can be applied, and the same order-wise results in Theorem \ref{them:informal} hold.  

\section{Proof of regret and constraint violation trade-off}

In this section, we will introduce technical assumptions, present the formal version of Theorem \ref{them:informal} and prove the main results.

\subsection{Preliminaries}

To analyze the algorithm, 
we introduce an ``$\epsilon$-tight'' ($\epsilon\geq 0$) optimization problem ($\epsilon$ is corresponding to tightness added in the original problem \eqref{obj-fluid}-\eqref{resource limit-fluid} and virtual queues \eqref{virtual queue}):
\begin{align}
    \max_{\mathbf x} & ~  \sum_{i,j} r_{i,j} x_{i,j} \label{obj-fluid-tight}\\
    \text{s.t.} & ~ \lambda_i = \sum_{j}  x_{i,j}, ~\forall i \in \mathcal N,  ~x_{i,j} \geq 0, \forall i \in \mathcal N, j \in \mathcal M, \label{arrival-fluid-tight} \\
    & ~ \sum_{i}  w_{i,j}^{(k)} x_{i,j} + \epsilon \leq \rho_j^{(k)}, ~\forall j \in \mathcal M, \forall k \in \mathcal K. \label{resource limit-fluid-tight} 
\end{align}
Let $\mathbf x^*_{\epsilon}$ and $\mathcal X_{\epsilon}$ be an optimal solution and feasible region in \eqref{obj-fluid-tight}-\eqref{resource limit-fluid-tight}, respectively. Further, we make the following mild assumptions.
\begin{assumption}
The reward $r_{i,j}(t)$ is a random variable in $[0, 1]$ with $\mathbb E[r_{i,j}(t)] = r_{i,j}$ for any $i \in \mathcal N, j \in \mathcal M, 0 \leq t < T.$ \label{assumption:reward}
\end{assumption}

\begin{assumption}
The arrival $\Lambda_{i}(t) \leq C_{\lambda}$ for any $i \in \mathcal N, 0 \leq t < T$ and $\mathbb E[\Lambda_{i}(t)] = \lambda_i, 0 \leq t < T.$ \label{assumption:arrival}
\end{assumption}

\begin{assumption}
The weights and requirements in the constraints 
satisfy $|w_{i,j}^{(k)}(t)|, |\rho_{j}^{(k)}(t)| \leq C_u$  $\forall k, i, j$ and $0 \leq t < T.$ 
\label{assumption:constraints}
\end{assumption}

\begin{assumption}
There exists $\delta > 0$ such that we can always find a feasible solution $\mathbf x \in \mathcal X$ to satisfy $\sum_{i}  w_{i,j}^{(k)} x_{i,j} - \rho_j^{(k)} \leq -\delta, \forall j \in \mathcal M, k \in \mathcal K.$ \label{assumption:slater}
\end{assumption}

Now we are ready to present our formal results in Theorem \ref{thm:formal-UCB} and \ref{thm:formal-MOSS}.

\begin{theorem}
Under Assumptions \ref{assumption:reward}-\ref{assumption:slater}, assuming $\delta \geq 4\epsilon$ with $\epsilon= \frac{2\sqrt{B\sqrt{MK}}}{\sqrt{T}}$ and let $V =\frac{\delta}{2\sum_i\lambda_i} \sqrt{\frac{TB}{\sqrt{MK}}},$ POND in Algorithm \ref{main alg} with UCB learning achieves the following regret and constraint violations bounds: 
\begin{align*}
&\sum_{t=0}^{T-1} \mathbb E\left[\sum_{i,j} r_{i,j} (x_{i,j}^*-x_{i,j}(t))\right] \leq \left(\frac{4}{\delta}\sqrt{B\sqrt{MK}}+8\sqrt{M\log T} \right) \sqrt{T}\sum_{i} \lambda_i=O\left(\sqrt{T\log T}\right), \\
&\sum_j\sum_{k} \left(\sum_{t=0}^{T-1}  \mathbb E\left[\sum_{i} w_{i,j}^{(k)}(t)x_{i,j}(t)-\rho_j^{(k)}(t)\right] \right)^+ \leq \left(\frac{3\nu_{\max}^2}{\gamma}\log\left(\frac{2\nu_{\max}}{\gamma}\right) + \nu_{\max}+\frac{4B}{\delta}\right)(MK)^{1.5}=O(1).
\end{align*}
where $B=MK (C_{\lambda}^{2}C_{u}^{2} + \epsilon^2),$ $\gamma = \frac{\delta}{2} - \epsilon$ and $\nu_{\max} = \max(\gamma, MKC_{\lambda}C_u).$
\label{thm:formal-UCB}
\end{theorem}

The regret bound of POND can be improved to be $O(\sqrt{T})$ with MOSS learning.
\begin{theorem}
Under Assumptions \ref{assumption:reward}-\ref{assumption:slater}, assuming $\delta \geq 4\epsilon$ with $\epsilon= \frac{2\sqrt{B\sqrt{MK}}}{\sqrt{T}}$ and let $V =\frac{\delta}{2\sum_i\lambda_i} \sqrt{\frac{TB}{\sqrt{MK}}},$ POND in Algorithm \ref{main alg} with MOSS learning achieves the following regret and constraint violations bounds: 
\begin{align*}
&\sum_{t=0}^{T-1} \mathbb E\left[\sum_{i,j} r_{i,j} (x_{i,j}^*-x_{i,j}(t))\right] \leq \left(\frac{4}{\delta}\sqrt{B\sqrt{MK}}+25\sqrt{M} \right) \sqrt{T}\sum_{i} \lambda_i=O\left(\sqrt{T}\right), \\
&\sum_j\sum_{k} \left(\sum_{t=0}^{T-1}  \mathbb E\left[\sum_{i} w_{i,j}^{(k)}(t)x_{i,j}(t)-\rho_j^{(k)}(t)\right] \right)^+ \leq \left(\frac{3\nu_{\max}^2}{\gamma}\log\left(\frac{2\nu_{\max}}{\gamma}\right) + \nu_{\max}+\frac{4B}{\delta}\right)(MK)^{1.5}=O(1).
\end{align*}
where $B=MK (C_{\lambda}^{2}C_{u}^{2} + \epsilon^2),$ $\gamma = \frac{\delta}{2} - \epsilon$ and $\nu_{\max} = \max(\gamma, MKC_{\lambda}C_u).$
\label{thm:formal-MOSS}
\end{theorem}
The bounds in Theorem \ref{thm:formal-MOSS} are sharp because the regret bound does not depend on the reward distributions so matches the (reward)-distribution independent regret $\Omega(\sqrt{T})$ in multi-armed bandit problems without constraints \cite{BubCes_12}; and $O(1)$ constraint violation is the smallest possible and order-wise optimal (note we can even achieve zero constraint violation for large $T$ by choosing a slightly large $\epsilon$ without affecting the order of the regret bound). 

\subsection{Outline of the Proofs of Theorem \ref{thm:formal-UCB} and \ref{thm:formal-MOSS}}

To outline the formal proofs for Theorem \ref{thm:formal-UCB} and \ref{thm:formal-MOSS}, we first provide an intuitive way to derive POND by using the Lyapunov-drift analysis (see, e.g. \cite{Sto_05,Nee_10,SriYin_14}), in particular, the drift-plus-penalty method \cite{Nee_10}. Define Lyapunov function to be $$L(t) = \frac{1}{2}\sum_j\sum_k \left(Q_j^{(k)}(t)\right)^2$$ and its drift to be $$\Delta(t) = L(t+1) - L(t).$$ Algorithm \ref{main alg} is such that it decides allocation $\mathbf x(t)$ to maximize the utility (reward) ``minus'' the Lyapunov drift at time slot $t:$
\begin{align}
     V\sum_{i,j} \hat r_{i,j}(t)x_{i,j}(t) - \Delta(t), \label{lyapunov-drift}
\end{align}
where parameter $V$ controls the trade-off between the utility (reward) and Lyapunov drift.

We note that after substituting the virtual queues update, the maximization problem above becomes the max-weight problem with the weight $V \hat r_{i,j}(t)- \sum_{k} w_{i,j}^{(k)}(t)Q_j^{(k)}(t)$ of a type-$i$ job to server $j$ in POND. Intuitively, a large $V$ gives preference to allocating jobs to servers with higher rewards; and a small $V$ gives preference to reducing virtual queue lengths (i.e., constraint violations). 

To analyze regret $\mathcal R(T)$, we decompose it as follows: 
\begin{align}
&T \sum_{i,j}r_{i,j}x_{i,j}^* - \sum_{t=0}^{T-1} \mathbb E\left[\sum_{i,j}r_{i,j}x_{i,j}(t)\right] \nonumber\\
=&   
\underbrace{T \sum_{i,j}r_{i,j}\left(x^*_{i,j} -x^*_{\epsilon,i,j}\right)}_{\text{$\epsilon$-tight}} + \underbrace{\sum_{t=0}^{T-1} \mathbb E\left[ \sum_{i,j}\left(r_{i,j}-\hat r_{i,j}(t)\right)x_{\epsilon,i,j}^*\right]}_{\text{reward mismatch}}+\underbrace{ \sum_{t=0}^{T-1}\mathbb E\left[ \sum_{i,j}\left(\hat r_{i,j}(t)-r_{i,j}\right)x_{i,j}(t)\right]}_{\text{reward mismatch}} \nonumber\\
&+\sum_{t=0}^{T-1} \mathbb E\left[ \sum_{i,j}\hat r_{i,j}(t)x_{\epsilon,i,j}^* \underbrace{- \sum_{i,j}\hat r_{i,j}(t)x_{i,j}(t) + \Delta(t)}_{\text{drift - reward}}\right] - \underbrace{\sum_{t=0}^{T-1} \mathbb E\left[\Delta(t)\right]}_{\text{accumulated drift}} \label{intuition}.
\end{align}
After the decomposition, we can see that 
\begin{itemize}
\item The first term is the difference between the regret when comparing the optimal offline (fluid) optimization and the ``$\epsilon$-tight'' offline (fluid) optimization. This term, $x^*_{i,j} -x^*_{\epsilon,i,j},$ adds $O(\epsilon)$ regret at each time slot, leading to $O(T\epsilon)$ regret over $T$ time slots. 

\item The second and third terms are related to a mismatch between the estimated reward and the actual reward received. This term is small because $\hat{\mathbf r}(t) \rightarrow \mathbf r$ with sufficient exploration. In particular, the UCB algorithm (or MOSS algorithm) guarantees that the cumulative mismatch over $T$ time slots is $O(\sqrt{T})$ regret.

\item The fourth term is on ``drift - reward''. POND aims to minimize $\frac{\Delta(t)}{V} - \sum_{i,j}\hat r_{i,j}(t)x_{i,j}(t),$ which is bounded by $\frac{\text{Constant}}{V} - \sum_{i,j}\hat r_{i,j}(t)x_{\epsilon,i,j}^*$ because $\mathbf x(t)$ is better than static $\mathbf x_{\epsilon}^*.$ Therefore, ``drift - reward'' adds $O\left(\frac{T}{V}\right)$ regret over $T$ time slots, which is $O(\sqrt{T})$ by choosing $V = O(\sqrt{T}).$

\item The last term is the accumulated drift, which is equal to $L(T)$  given $Q(0) = 0$ and is positive.
\end{itemize}

The constraint violations by time $T$ can be bounded by the expected virtual queue lengths at time $T,$ in particular, 
\begin{align*}
\sum_{t=0}^{T-1}  \mathbb E\left[\sum_{i} w_{i,j}^{(k)}(t)x_{i,j}(t)-\rho_j^{(k)}(t)\right] \leq \mathbb E[ Q_j^{(k)}(T)]-T\epsilon, ~\forall j, k.  
\end{align*}
We will use the Lyapunov drift analysis to show that $\mathbb E[Q_j^{(k)}(T)]$ is bounded by $O(V).$ Therefore, by choosing proper  $V = O(\sqrt{T})$ and $\epsilon = O(1/\sqrt{T}),$ we can reduce the violation to be $O(1)$ at the expense of adding $O(T\epsilon) =O(\sqrt{T})$ to  total regret (due to the first term in \eqref{intuition}).
Moreover, it is well known that $\Omega(\sqrt{T})$ regret is a lower bound in multi-armed bandit without constraints \cite{LatSze_20}, which implies that the regret bound $O(\sqrt{T})$ in Theorem \ref{thm:formal-MOSS} is optimal. 

Next, we provide the detailed proof of Theorem \ref{thm:formal-UCB} and \ref{thm:formal-MOSS}, where we perform Lyapunov drift analysis to bridge the regret and constraint violations.  

\subsection{Lyapunov Drift}
Define Lyapunov function to be $$L(t) = \frac{1}{2}\sum_j\sum_k \left(Q_j^{(k)}(t)\right)^2.$$ The Lyapunov drift 
\begin{align*}
&L(t+1)-L(t) \\
\leq&  \sum_j\sum_{k}Q_j^{k}(t)\left(\sum_{i}w_{i,j}^{(k)}(t) x_{i,j}(t) - \rho_j^{(k)}(t)+\epsilon\right) + \frac{\sum_{j}\sum_{k} \left(\sum_{i}w_{i,j}^{(k)}(t) x_{i,j}(t) - \rho_j^{(k)}(t)+\epsilon\right)^2}{2}\\
\leq& \sum_j\sum_{k}Q_j^{k}(t)\left(\sum_{i}w_{i,j}^{(k)}(t) x_{i,j}(t) - \rho_j^{(k)}(t)+\epsilon\right) -  V \sum_{i,j} \hat r_{i,j}(t)x_{i,j}(t) + V \sum_{i,j} \hat r_{i,j}(t)x_{i,j}(t) + B
\end{align*}
where the first inequality holds because $ Q_j^{(k)}(t+1) = \left[Q_j^{(k)}(t) + \sum_{i}w_{i,j}^{(k)}(t) x_{i,j}(t) - \rho_j^{(k)}(t)+ \epsilon \right]^{+}$ and the second inequality holds because $B=MK(C_{\lambda}^{2}C_{u}^{2} + \epsilon^2).$ We study the expected drift conditioned on the current state $\mathbf H(t)=[\mathbf Q(t),  \hat{\mathbf r}(t)]=\mathbf h=[\mathbf Q, \hat{\mathbf r}],$ including the virtual queues $\mathbf Q$ and the learned reward $\hat{\mathbf r}$ at time slot $t.$
\begin{align}
& \mathbb E[L(t+1)-L(t)|\mathbf H(t)=\mathbf h] \nonumber\\
\leq& \mathbb E\left[\sum_j\sum_{k}Q_j^{(k)}\left(\sum_{i}w_{i,j}^{(k)}(t) x_{i,j}(t) - \rho_j^{(k)}(t)+\epsilon\right)  -  V \sum_{i,j} \hat r_{i,j}(t)x_{i,j}(t) \Big|\mathbf H(t)=\mathbf h\right] \nonumber\\
&+ V \mathbb E\left[\sum_{i,j} \hat r_{i,j}(t)x_{i,j}(t)\Big|\mathbf H(t)=\mathbf h\right] + B \nonumber\\
\leq&\mathbb E\left[\sum_j\sum_{k}Q_j^{(k)}\left(\sum_{i}w_{i,j}^{(k)}(t) x_{\epsilon,i,j} - \rho_j^{(k)}(t)+\epsilon\right)  -  V \sum_{i,j} \hat r_{i,j}(t)x_{\epsilon,i,j}\Big|\mathbf H(t)=\mathbf h\right] \nonumber\\
&+ V \mathbb E\left[\sum_{i,j} \hat r_{i,j}(t)x_{i,j}(t)\Big|\mathbf H(t)=\mathbf h\right] + B \nonumber\\
=& \sum_j\sum_{k}Q_j^{(k)}\left(\sum_{i}w_{i,j}^{(k)} x_{\epsilon,i,j} - \rho_j^{(k)}+\epsilon\right) -  V\mathbb E\left[\sum_{i,j} \hat r_{i,j}(t)x_{\epsilon,i,j}-\sum_{i,j} \hat r_{i,j}(t)x_{i,j}(t) \Big| \mathbf H(t)=\mathbf h\right] + B. \label{drif ucb}
\end{align}
where $\mathbf x_{\epsilon}$ is a feasible solution to \eqref{obj-fluid-tight}-\eqref{resource limit-fluid-tight} with the $(i,j)^{\text{th}}$ entry $x_{\epsilon,i,j}$ and the second inequality holds according to Lemma \ref{lemma: drift-inequality} below; the last equality holds because $w_{i,j}^{(k)}(t)$ and $\rho_j^{(k)}(t)$ are independent with $\hat r_{i,j}(t).$ Based on Lyapunov drift analysis above, we investigate the regret and constraint violations in the two following subsections.
\begin{restatable}{lemma}{lemmadriftinequality}\label{lemma: drift-inequality}
\label{LEMMA: DRIFT-INEQUALITY}
Given any $\mathbf x_{\epsilon} \in \mathcal X_{\epsilon}$ and $\mathbf x(t)$ is the solution in Algorithm \ref{main alg}, we have
\begin{align}
&\mathbb E\left[\sum_j\sum_{k}Q_j^{(k)}\left(\sum_{i}w_{i,j}^{(k)}(t) x_{i,j}(t) - \rho_j^{(k)}(t)+\epsilon\right)  -  V \sum_{i,j} \hat r_{i,j}(t)x_{i,j}(t)\Big|\mathbf H(t)=\mathbf h\right] \nonumber\\
\leq&\mathbb E\left[\sum_j\sum_{k}Q_j^{(k)}\left(\sum_{i}w_{i,j}^{(k)}(t) x_{\epsilon,i,j} - \rho_j^{(k)}(t)+\epsilon\right)  -  V \sum_{i,j} \hat r_{i,j}(t)x_{\epsilon,i,j}\Big|\mathbf H(t)=\mathbf h\right] \nonumber
\end{align}
\end{restatable}
The proof of Lemma \ref{lemma: drift-inequality} can be found in Appendix \ref{app: lemma-drift-inequality}.

\subsection{Regret}\label{sec: regret}
Let $x_{\epsilon,i,j} = x_{\epsilon,i,j}^*$ in the drift analysis \eqref{drif ucb} and it implies  
\begin{align*}
& \mathbb E[L(t+1)-L(t)|\mathbf H(t)=\mathbf h] +  V \sum_{i,j}r_{i,j}x^*_{i,j} - V \mathbb E\left[\sum_{i,j}r_{i,j}x_{i,j}(t)\Big|\mathbf H(t)=\mathbf h\right] \\
\leq& \sum_j\sum_{k}Q_j^{(k)}\left(\sum_{i}w_{i,j}^{(k)} x_{\epsilon,i,j}^* - \rho_j^{(k)}+\epsilon\right)+ B + V \sum_{i,j}r_{i,j}x^*_{i,j} - V \sum_{i,j}r_{i,j}x^*_{\epsilon,i,j}\\
&+ V \mathbb E\left[\sum_{i,j}\hat r_{i,j}(t)x_{i,j}(t)\Big|\mathbf H(t)=\mathbf h\right] - V \mathbb E\left[\sum_{i,j}r_{i,j}x_{i,j}(t)\Big|\mathbf H(t)=\mathbf h\right]  \\
&+ V \sum_{i,j}r_{i,j}x^*_{\epsilon,i,j} - V \mathbb E\left[\sum_{i,j}\hat r_{i,j}(t)x_{\epsilon,i,j}^*\Big|\mathbf H(t)=\mathbf h\right], 
\end{align*}
where we added the regret $\sum_{i,j}r_{i,j}x^*_{i,j} - \mathbb E\left[\sum_{i,j}r_{i,j}x_{i,j}(t)|\mathbf H(t)=\mathbf h\right]$ to the drift,  and then added and subtracted $V \sum_{i,j}r_{i,j}x^*_{\epsilon,i,j}.$ Note $Q(0) = 0, Q(t) \geq 0, \forall 0\leq t < T$ and  $\left(\sum_{i}w_{i,j}^{(k)} x_{\epsilon,i,j}^* - \rho_j^{(k)}+\epsilon\right)$ is non-positive because $x^*_{\epsilon}$ is feasible solution to \eqref{obj-fluid-tight}-\eqref{resource limit-fluid-tight}.
Taking the expected value with respect to $\mathbf H(t),$ doing the telescope summation across time up to $T-1$ and dividing $V$ both sides imply
\begin{align}
 T \sum_{i,j}r_{i,j}x_{i,j}^* - \sum_{t=0}^{T-1} \mathbb E\left[\sum_{i,j}r_{i,j}x_{i,j}(t)\right] \leq&   
 T \sum_{i,j}r_{i,j}x^*_{i,j} - T \sum_{i,j}r_{i,j}x^*_{\epsilon,i,j} + \frac{TB}{V} \label{epsilon-bound}\\ 
&+ \sum_{t=0}^{T-1}\mathbb E\left[ \sum_{i,j}\hat r_{i,j}(t)x_{i,j}(t) -  \sum_{i,j}r_{i,j}x_{i,j}(t)\right] \label{UCB-bound} \\
&+ \sum_{t=0}^{T-1} \mathbb E\left[ \sum_{i,j}r_{i,j}x_{\epsilon,i,j}^* -  \sum_{i,j}\hat r_{i,j}(t)x_{\epsilon,i,j}^*\right], \label{UCB-negative-bound}
\end{align}
where the upper bound on the regret consists of three major terms: $  T \sum_{i,j}r_{i,j}x^*_{i,j} - T \sum_{i,j}r_{i,j}x^*_{\epsilon,i,j}$ in \eqref{epsilon-bound} is  the difference between the original optimization problem \eqref{obj-fluid}-\eqref{resource limit-fluid} and the ``$\epsilon$-tight'' optimization problem \eqref{obj-fluid-tight}-\eqref{resource limit-fluid-tight}; \eqref{UCB-bound} is the difference between estimated rewards and actual rewards under allocation ${\bf x}(t)$,  and \eqref{UCB-negative-bound} is the difference between estimated rewards and the actual reward under allocation ${\bf x}_\epsilon^*$. In the following, we introduce three lemmas to bound the above three terms, respectively. 

The first lemma is on the difference term in \eqref{epsilon-bound} and we prove it is $O(T\epsilon).$
\begin{restatable}{lemma}{lemmaelsilongap}\label{lemma: elsilon gap}
Let $\mathbf x^*$ the optimal solution to \eqref{obj-fluid}-\eqref{resource limit-fluid}  and $\mathbf x^*_{\epsilon}$ the optimal solution to \eqref{obj-fluid-tight}-\eqref{resource limit-fluid-tight}. Under Assumptions \ref{assumption:reward}-\ref{assumption:slater}, we have 
$$T \sum_{i,j}r_{i,j}x^*_{i,j} - T \sum_{i,j}r_{i,j}x^*_{\epsilon,i,j} \leq \frac{T\epsilon}{\delta} \sum_i \lambda_i.$$\label{epsilon-lemma}
\end{restatable}

The second lemma is on the term \eqref{UCB-bound} and we show it is $O(\sqrt{MT})$ with MOSS learning and $O(\sqrt{MT\log T})$ with UCB learning.
\begin{lemma}
With MOSS and UCB learning, we have upper bound on \eqref{UCB-bound} that 
\begin{itemize}
    \item MOSS learning:
\begin{align*}
\eqref{UCB-bound} \leq  21\sqrt{MT} \sum_{i} \lambda_i;
\end{align*}
    \item UCB learning:
\begin{align*}
\eqref{UCB-bound} \leq  7\sqrt{MT\log T} \sum_{i} \lambda_i.
\end{align*}\label{MAB-lemma}
\end{itemize}
\end{lemma}

The third lemma is on the term \eqref{UCB-negative-bound} and we show it is $O(\sqrt{T})$ for MOSS learning and $O(1)$ for UCB learning.
\begin{lemma} 
With MOSS and UCB learning, we have upper bound on \eqref{UCB-negative-bound} that
\begin{itemize}
    \item MOSS learning:
\begin{align*}
\eqref{UCB-negative-bound} \leq  \left(3\sqrt{MT}+6M\right) \sum_{i}\lambda_i;
\end{align*}
    \item UCB learning:
\begin{align*}
\eqref{UCB-negative-bound} \leq 2C_{\lambda}\sum_{i} \lambda_i.
\end{align*}
\end{itemize}
\label{MAB-negative-lemma}
\end{lemma}

Based on lemmas \ref{epsilon-lemma}, \ref{MAB-lemma} and \ref{MAB-negative-lemma}, we conclude that
\begin{itemize}
\item MOSS learning:
\begin{align}
 T \sum_{i,j}r_{i,j}x_{i,j}^* - \sum_{t=0}^{T-1} \mathbb E\left[\sum_{i,j}r_{i,j}x_{i,j}(t)\right] \leq& \frac{T\epsilon}{\delta}\sum_i \lambda_i + \frac{TB}{V} + 25\sqrt{MT} \sum_{i} \lambda_i.
\end{align}
\item UCB learning:
\begin{align}
 T \sum_{i,j}r_{i,j}x_{i,j}^* - \sum_{t=0}^{T-1} \mathbb E\left[\sum_{i,j}r_{i,j}x_{i,j}(t)\right] \leq&  \frac{T\epsilon}{\delta}\sum_i \lambda_i + \frac{TB}{V} + 8\sqrt{MT\log T} \sum_{i} \lambda_i.
\end{align}
\end{itemize}

\subsection{Constraint Violations}\label{sec: cv}
According to virtual queue update in Algorithm \ref{main alg}, we have 
\begin{align*}
    Q_j^{(k)}(t+1) =& \left[Q_j^{(k)}(t) + \sum_{i}w_{i,j}^{(k)}(t) x_{i,j}(t) - \rho_j^{(k)}(t) + \epsilon \right]^{+}\\
    \geq&Q_j^{(k)}(t) + \sum_{i}w_{i,j}^{(k)}(t) x_{i,j}(t) - \rho_j^{(k)}(t) + \epsilon,
\end{align*}
which implies $$\sum_{i}w_{i,j}^{(k)}(t) x_{i,j}(t) - \rho_j^{(k)}(t) + \epsilon \leq Q_j^{(k)}(t+1) - Q_j^{(k)}(t).$$
Summing over $t,$ we have
\begin{align}
\left(\sum_{t=0}^{T-1}  \mathbb E\left[\sum_{i} w_{i,j}^{(k)}(t)x_{i,j}(t) - \rho_j^{(k)}(t)\right]\right)^+ \leq \left(\mathbb E\left[ Q_j^{(k)}(T)\right]-T\epsilon\right)^{+}, \nonumber
\end{align}
which implies
\begin{align}
\sum_{j}\sum_{k} \left(\sum_{t=0}^{T-1}  \mathbb E\left[\sum_{i} w_{i,j}^{(k)}(t)x_{i,j}(t) - \rho_j^{(k)}(t)\right]\right)^+ \leq \sum_{j}\sum_{k} \left(\mathbb E\left[ Q_j^{(k)}(T)\right]-T\epsilon\right)^{+}, \label{queue and vioaltion}
\end{align} where $Q_j^{(k)}(0) = 0, \forall j, k.$ 

As shown in \eqref{queue and vioaltion}, we are able to bound constraint violations by $\left(\mathbb E[Q_j^{(k)}(T)]-T\epsilon\right)^{+}.$ We establish the upper bound on constraint violations in the following lemma by bounding $\mathbb E[Q_j^{(k)}(T)], \forall j, k.$
\begin{lemma}\label{lemma: cons violation} 
\begin{align*}
&\left(\sum_{t=0}^{T-1}  \mathbb E\left[\sum_{i} w_{i,j}^{(k)}(t)x_{i,j}(t)-\rho_j^{(k)}(t)\right] \right)^+ \\
&\leq \frac{3\sqrt{MK}\nu_{\max}^2}{\gamma}\log\left(\frac{2\nu_{\max}}{\gamma}\right)+\sqrt{MK}\nu_{\max}+\frac{4\sqrt{MK}(V \sum_i \lambda_i + B)}{\delta}-T\epsilon,
\end{align*}
where $\gamma = \frac{\delta}{2} - \epsilon$ and $\nu_{\max} = \max(\gamma, MKC_{\lambda}C_u).$
\end{lemma}
This result implies that the constraint violations are bounded by a constant that depends on $K, C_\lambda, C_u, \delta, \epsilon$, $B$ and $V.$

The key to prove Lemma \ref{lemma: cons violation} is to establish an upper bound on $\mathbb E[\sum_j\sum_k Q_j^{(k)}(T)].$  
We next present a lemma that can be used to bound $\mathbb E[\sum_j\sum_k Q_j^{(k)}(T)],$ which is derived in \cite{Nee_16} and could be regarded as an application of \cite{Haj_82}.
\begin{lemma} \label{drif lemma}
Let $S(t)$ be the state of Markov chain, $L(t)$ be a Lyapunov function and its drift $\Delta(t) = L(t+1)-L(t).$ Given the constants $\gamma$ and $\nu_{\max}$ with $0 < \gamma \leq \nu_{\max},$ suppose the expected drift $\mathbb E[\Delta(t) | S(t)=s]$ satisfies the following conditions:

(i) There exists constants $\gamma > 0$ and $\theta > 0$ such that $\mathbb E[\Delta(t)|S(t)=s] \leq -\gamma$ when $V(t) \geq \theta.$

(ii) $|L(t+1) - L(t)| \leq \nu_{\max}$ holds with probability one.

Then we have
\begin{align*}
    \mathbb E[e^{r L(t)}] 
    \leq 1 + \frac{2e^{r(\nu_{\max}+\theta)}}{r\gamma},
\end{align*}
where $r = \frac{\gamma}{\nu_{\max}^2+\nu_{\max}\gamma/3}.$
\end{lemma}

Though we have analyzed the conditional expected drift of Lyapunov function $\sum_j\sum_{k} \left(Q_j^{(k)}(t)\right)^2$ in Lemma \ref{drif lemma}, we note that cannot be used as a Lyapunov function because condition ii) is not satisfied. Therefore, we consider Lyapunov function \begin{align*}
\bar L(t) = \sqrt{\sum_j\sum_{k} \left(Q_j^{(k)}(t)\right)^2} = ||\mathbf Q(t)||_2
\end{align*} as in \cite{ErySri_12} and prove conditions i) and ii) in Lemma \ref{drif lemma} for $\bar L(t)$ are satisfied in the following subsection.

\subsubsection{\textbf{Drift condition}} 
Given $\mathbf H(t)=\mathbf h$ and $\bar L(t)\geq \theta = \frac{4(V\sum_i\lambda_i + B)}{\delta},$ the conditional expected drift of $\bar L(t)$ is 
\begin{align*}
&\mathbb E[ ||\mathbf Q(t+1)||_2 - ||\mathbf Q(t)||_2 | \mathbf H(t) = \mathbf h] \\
=& \mathbb E\left[ \sqrt{||\mathbf Q(t+1)||_2^2} - \sqrt{||\mathbf Q(t)||_2^2} | \mathbf H(t) = \mathbf h\right] \\
\leq & \frac{1}{2||\mathbf Q||_2} \mathbb E[ ||\mathbf Q(t+1)||_2^2 - ||\mathbf Q(t)||_2^2 | \mathbf H(t) = \mathbf h] \\
\leq & -\left(\frac{3\delta}{4} - \epsilon\right)\frac{||\mathbf Q||_1}{||\mathbf Q||_2} + \frac{V \sum_i\lambda_i + B}{||\mathbf Q||_2} \\
\leq&  -\left(\frac{3\delta}{4} - \epsilon\right) + \frac{V \sum_i\lambda_i + B}{||\mathbf Q||_2} \\
\leq&  -\left(\frac{3\delta}{4} - \epsilon\right) + \frac{V \sum_i\lambda_i + B}{\theta} = -\left(\frac{\delta}{2} -\epsilon\right)
\end{align*}
where the first inequality holds because $\sqrt{x}$ is a concave function; the second inequality holds by Lemma \ref{negative drift} below; the third inequality holds because $||\mathbf Q||_1 \geq ||\mathbf Q||_2;$ and the last inequality holds given $||\mathbf Q||_2 \geq \theta.$
\begin{restatable}{lemma}{driflemma} 
\label{negative drift}
Let $L(t) = \frac{1}{2}\sum_j\sum_{k} \left(Q_j^{(k)}(t)\right)^2.$ Under POND algorithm, the conditional expected drift is  
\begin{align*}
\mathbb E[L(t+1)-L(t)|\mathbf H(t)=\mathbf h] \leq-\left(\frac{3\delta}{4} - \epsilon\right)\sum_j\sum_{k}Q_j^{(k)}(t) + V \sum_i\lambda_i + B.  
\end{align*}
\end{restatable}
The proof of the lemma can be found in Appendix \ref{app: drift-lemma}.

Moreover, for condition ii) in Lemma \ref{drif lemma}, we have 
\begin{align*}
||\mathbf Q(t+1)||_2 - ||\mathbf Q(t)||_2 \leq ||\mathbf Q(t+1) - \mathbf Q(t)||_2 \leq ||\mathbf Q(t+1) - \mathbf Q(t)||_1 \leq M K C_{\lambda}C_u,
\end{align*}
where the last inequality holds because $|Q_j^{(k)}(t+1)-Q_j^{(k)}(t)| \leq C_{\lambda}C_u, \forall j, k,$ based on Assumption \ref{assumption:arrival} and \ref{assumption:constraints}.

\subsubsection{\textbf{Establishing bounds on $\mathbb E[\sum_j\sum_{k} Q_j^{(k)}(t)]$}}\label{sec: bounds on queue}

Let $\theta = \frac{4(V \sum_i\lambda_i + B)}{\delta},$ $\gamma = \frac{\delta}{2} - \epsilon,$ and $\nu_{\max} = \max(\gamma, MKC_{\lambda}C_u).$ We apply Lemma \ref{drif lemma} for $\bar L(t)$ and obtain 
\begin{align*}
    \mathbb E\left[e^{r ||\mathbf Q(t)||_2}\right] 
    \leq 1 + \frac{2e^{r(\nu_{\max}+\theta)}}{r\gamma}, ~ r = \frac{\gamma}{\nu_{\max}^2 + \nu_{\max}\gamma/3},
\end{align*}
which implies that
\begin{align*}
    \mathbb E\left[e^{\frac{r}{\sqrt{MK}} ||\mathbf Q(t)||_1}\right] 
    \leq 1 + \frac{2e^{r(\nu_{\max}+\theta)}}{r\gamma},
\end{align*} because $||\mathbf Q(t)||_1 \leq \sqrt{MK}||\mathbf Q(t)||_2.$
By Jensen's inequality, we have 
$$e^{\frac{r}{\sqrt{MK}} \mathbb E\left[||\mathbf Q(t)||_1\right]} \leq \mathbb E\left[e^{\frac{r}{\sqrt{MK}} ||\mathbf Q(t)||_1}\right],$$ which implies
\begin{align*}
    \mathbb E\left[||\mathbf Q(t)||_1\right]
    \leq& \frac{\sqrt{MK}}{r}\log\left(1+\frac{2e^{r(\nu_{\max}+\theta)}}{r\gamma}\right) \\ 
    \leq& \frac{\sqrt{MK}}{r}\log\left(1+\frac{8\nu_{\max}^2}{3\gamma^2}e^{r(\nu_{\max}+\theta)}\right) \\
    \leq& \frac{\sqrt{MK}}{r}\log\left(\frac{11\nu_{\max}^2}{3\gamma^2}e^{r(\nu_{\max}+\theta)}\right)\\
    \leq& \frac{3\sqrt{MK}\nu_{\max}^2}{\gamma}\log\left(\frac{2\nu_{\max}}{\gamma}\right)+\sqrt{MK}\nu_{\max}+\sqrt{MK}\theta\\
    =& \frac{3\sqrt{MK}\nu_{\max}^2}{\gamma}\log\left(\frac{2\nu_{\max}}{\gamma}\right)+\sqrt{MK}\nu_{\max}+\frac{4\sqrt{MK}(V\sum_i\lambda_i + B)}{\delta}
\end{align*}
where the second, third and fourth inequalities hold because $r = \frac{\gamma}{\nu_{\max}^2+\nu_{\max}\gamma/3}$ and $0 < \gamma \leq \nu_{\max}.$

\subsection{Proving Theorem \ref{thm:formal-UCB} and \ref{thm:formal-MOSS}.}
Based on the results in subsections \ref{sec: regret} and \ref{sec: cv} above, we obtain the regret and constraint violations in Theorem \ref{thm:formal-MOSS} 
\begin{align*}
\sum_{t=0}^{T-1} \mathbb E\left[\sum_{i,j}r_{i,j}(x_{i,j}^*-x_{i,j}(t))\right] \leq&  \frac{T\epsilon}{\delta}\sum_i\lambda_i + \frac{TB}{V} + 25\sqrt{MT} \sum_{i} \lambda_i, \\
\left(\sum_{t=0}^{T-1}  \mathbb E\left[\sum_{i} w_{i,j}^{(k)}(t)x_{i,j}(t)-\rho_j^{(k)}(t)\right] \right)^+ 
\leq& \frac{3\sqrt{MK}\nu_{\max}^2}{\gamma}\log\left(\frac{2\nu_{\max}}{\gamma}\right)+\sqrt{MK}\nu_{\max}+\frac{4\sqrt{MK}(V\sum_i\lambda_i + B)}{\delta}-T\epsilon.
\end{align*}
Let $\epsilon = \frac{4V\sqrt{MK}\sum_i\lambda_i}{T\delta},$ the constraint violations satisfy
\begin{align*}
\left(\sum_{t=0}^{T-1}  \mathbb E\left[\sum_{i} w_{i,j}^{(k)}(t)x_{i,j}(t)-\rho_j^{(k)}(t)\right] \right)^+ \leq \frac{3\sqrt{MK}\nu_{\max}^2}{\gamma}\log\left(\frac{2\nu_{\max}}{\gamma}\right) + \sqrt{MK}\nu_{\max} + \frac{4\sqrt{MK}B}{\delta}.
\end{align*}
Theorem \ref{thm:formal-MOSS} holds by choosing $V =\frac{\delta}{2\sum_i\lambda_i} \sqrt{\frac{TB}{\sqrt{MK}}}.$ Theorem \ref{thm:formal-UCB} can be proved by following the same steps.

\section{Experiments}\label{sec: experiments}
In this section, we present simulations, using synthetic experiments and a real dataset from online tutoring, which demonstrate the performance of our POND algorithm. 
In particular, we show that POND achieves the $O(\sqrt{T})$ regret and $O(1)$ constraint violations with ``tightness'', while without tightness, the algorithm achieves $O(\sqrt{T})$ regret and $O(\sqrt{T})$ constraint violations. We also see that POND outperforms the Explore-Then-Commit algorithm (baseline) significantly in the experiments.

As mentioned in Subsection \ref{sec: pond}, UCB performs better than MOSS in practice. So we use $\hat r_{i,j}(t) = \bar r_{i,j}(t-1) + \sqrt{\frac{\log{T}}{N_{i,j}(t-1)}},$ instead of MOSS learning. Our baseline is the Explore-Then-Commit algorithm, which uses the same UCB to explore for $NM\log T$ time slots to estimate the parameters (e.g. reward $\hat{\mathbf r}$ and arrival $\hat{\boldsymbol \lambda}$), solves \eqref{obj-fluid}-\eqref{resource limit-fluid} with the estimated parameters to obtain $\{\hat x_{i,j}^*\},$ and then uses $\hat x_{i,j}^*/\hat \lambda_i$ as the probability to allocate the incoming $i^{\text{th}}$ job to the server $j.$

\subsection{Synthetic Example}
We considered a model with two types of jobs and four servers. In particular, we assumed geometric arrivals with mean $\boldsymbol \lambda = [1.0, 2.0],$ Bernoulli rewards with mean $\mathbf r =\begin{bmatrix}
0.5 & 0.6 & 0.1 & 0.2\\
0.2 & 0.6 & 0.5 & 0.2
\end{bmatrix},$  capacity constraints $\sum_{i=1}^2x_{i,j} \leq \mu_j$ with $\boldsymbol \mu = [0.85, 0.85, 0.8, 0.8];$
fairness constraints $\sum_{i=1}^2 x_{i,j} \geq d_j\sum_{i=1}^2 \lambda_i$ with $\mathbf d = [0.25, 0.25, 0.20, 0.20],$
resource constraints $\sum_{i=1}^2w_{i,j }x_{i,j} \leq \rho_j$ with $\mathbf w = \begin{bmatrix}
2 & 2 & 2 & 2\\
4 & 4 & 4 & 3.5
\end{bmatrix}$
and $\boldsymbol \rho = [3, 3, 2.5, 2.5].$ Let $V=2\sqrt{T}$. We compared POND with ``tightness'' $\epsilon=O(1/\sqrt{T}),$ and ``no tightness'', i.e. $\epsilon=0.$ 

We simulated POND and Explore-Then-Commit over $T$ time slots with $T=[50^2, 75^2, 100^2, 125^2, 150^2],$ where $500$ independent trials were averaged for each $T.$ We plotted the regret, capacity violation, fairness violation and resource violation against $\sqrt{T}$ in Figure \ref{fig:regret and cv v.s. time}, where we used the \emph{maximum} average violation among four servers for each type of constraint violations. Figure \ref{fig:regret and cv v.s. time} shows that using POND with tightness constants $\epsilon = 0.5/\sqrt{T}$ and $1/\sqrt{T}$, POND achieved $O(\sqrt{T})$ regret as in Figure \ref{fig:regret and cv v.s. time-regret} and $O(1)$ constraint violation as in Figure \ref{fig:regret and cv v.s. time-cv}-\ref{fig:regret and cv v.s. time-rv}. Without the tightness constant, POND achieved $O(\sqrt{T})$ regret but $O(\sqrt{T})$ constraint violation as shown by the orange curve in Figure~\ref{fig:regret and cv v.s. time-cv}. These numerical results are consistent with our theoretical analysis.
The experimental results also show that using the tightness constant is critical to achieve the $O(1)$ constraint violations. Moreover, POND performed much better than Explore-Then-Commit by achieving both lower regret and lower constraint violations. 

\begin{figure}[h]
\centering
\begin{subfigure}{.24\textwidth}
  \centering
  \includegraphics[width=1.0\linewidth]{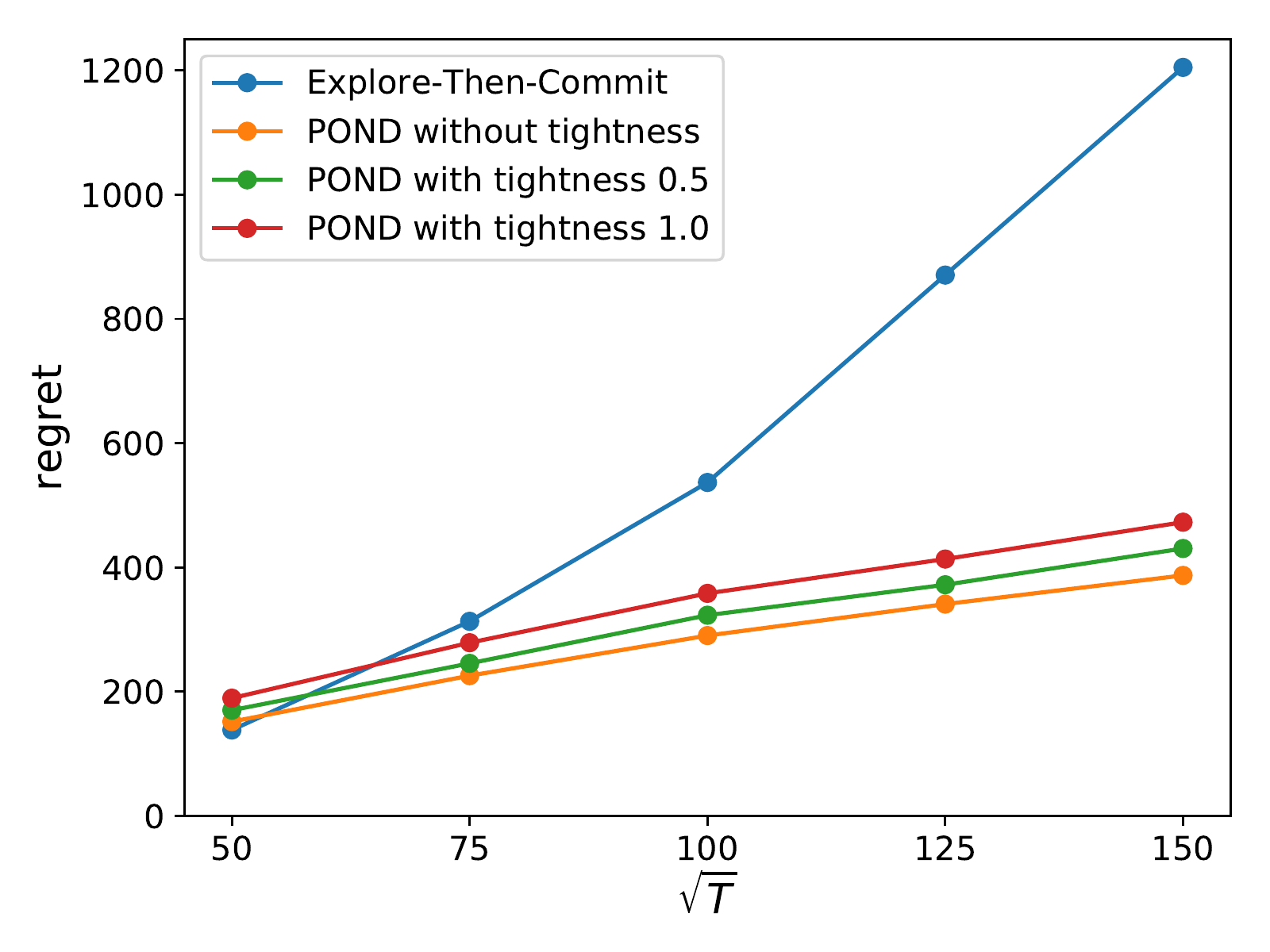}
  \caption{Regret}
  \label{fig:regret and cv v.s. time-regret}
\end{subfigure}%
\begin{subfigure}{.24\textwidth}
  \centering
  \includegraphics[width=1.0\linewidth]{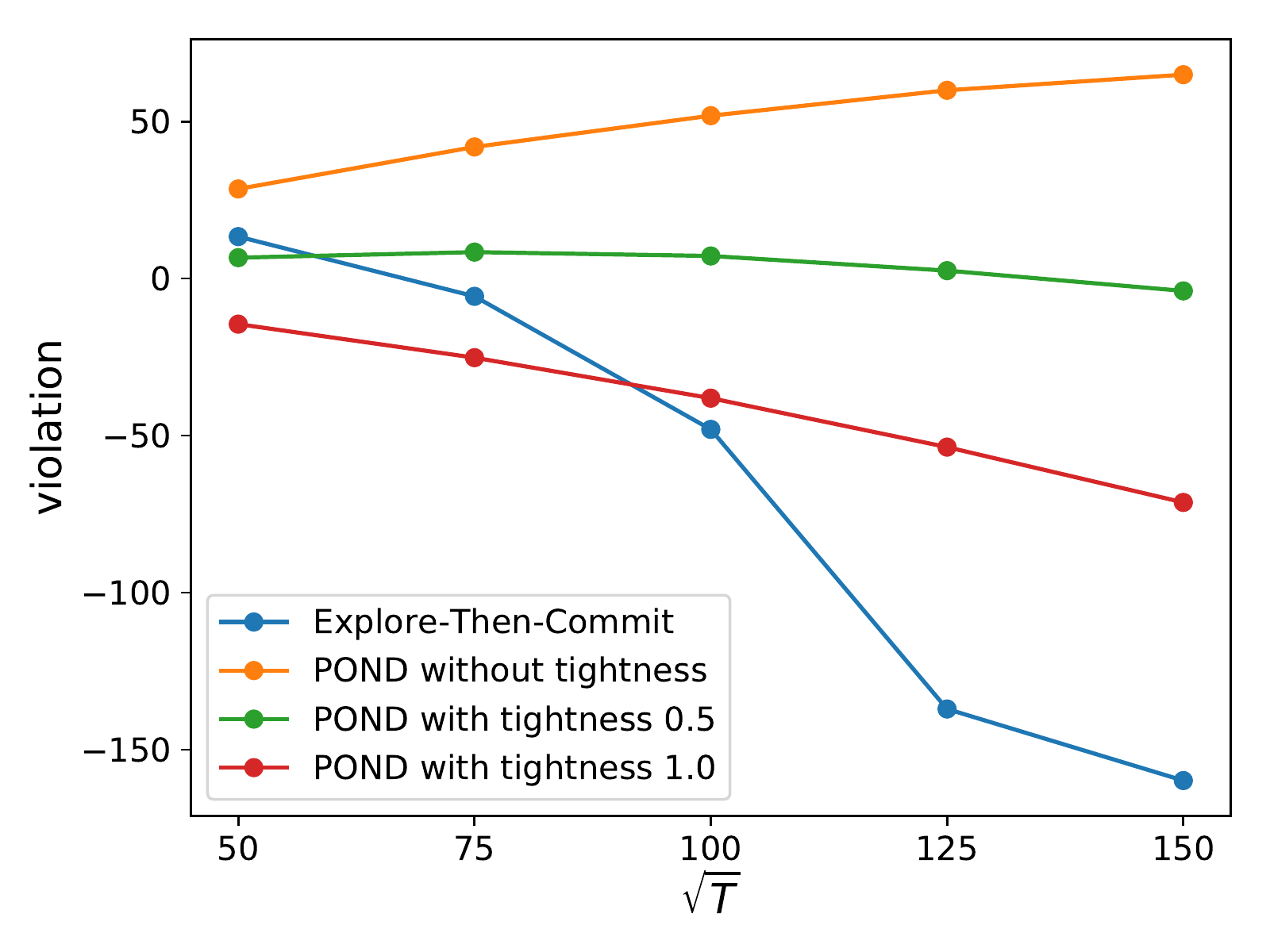}
  \caption{Capacity violation}
    \label{fig:regret and cv v.s. time-cv}
\end{subfigure}
\begin{subfigure}{.24\textwidth}
  \centering
  \includegraphics[width=1.0\linewidth]{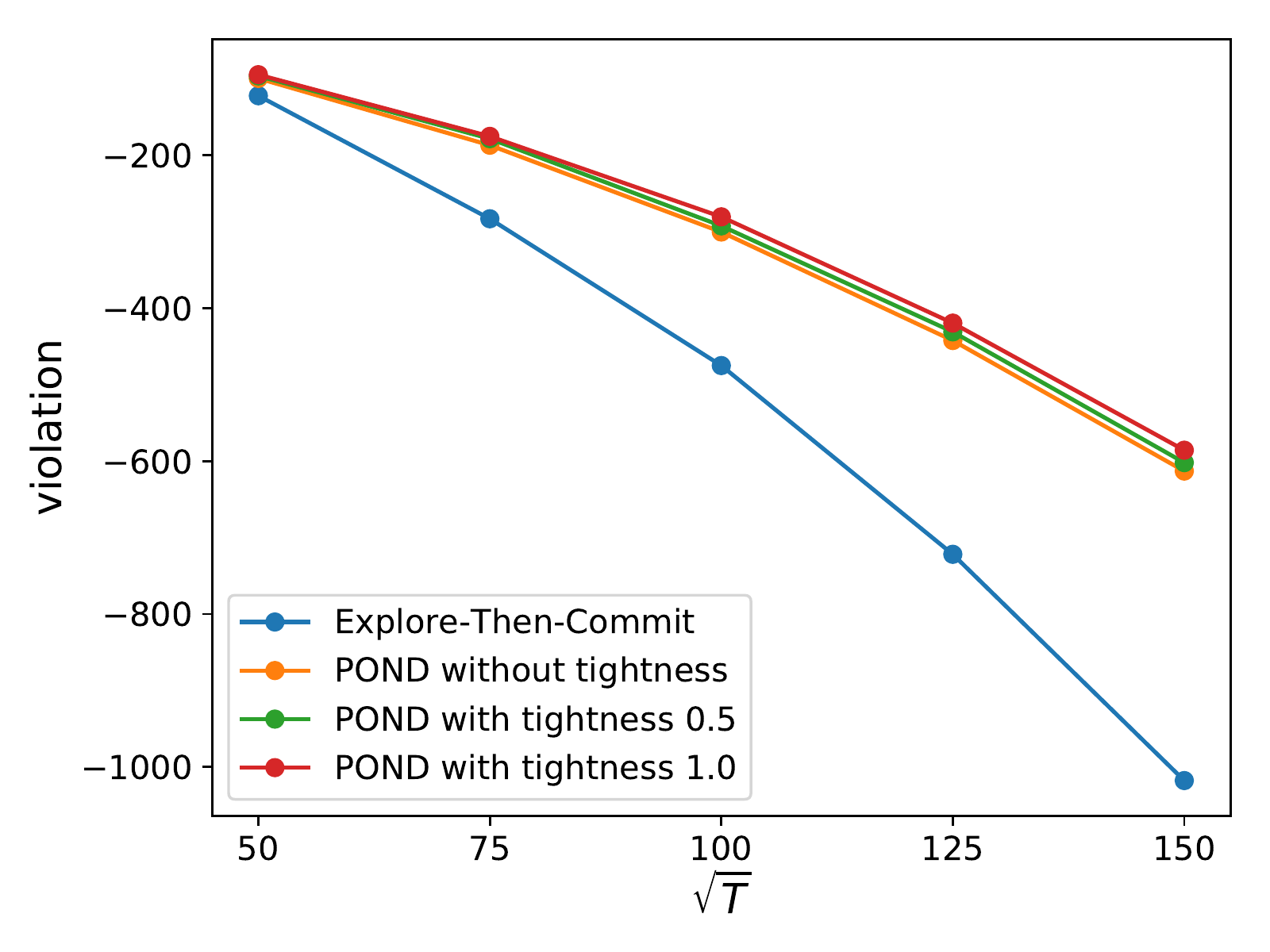}
  \caption{Fairness violation}
  \label{fig:regret and cv v.s. time-fv}
\end{subfigure}
\begin{subfigure}{.24\textwidth}
  \centering
  \includegraphics[width=1.0\linewidth]{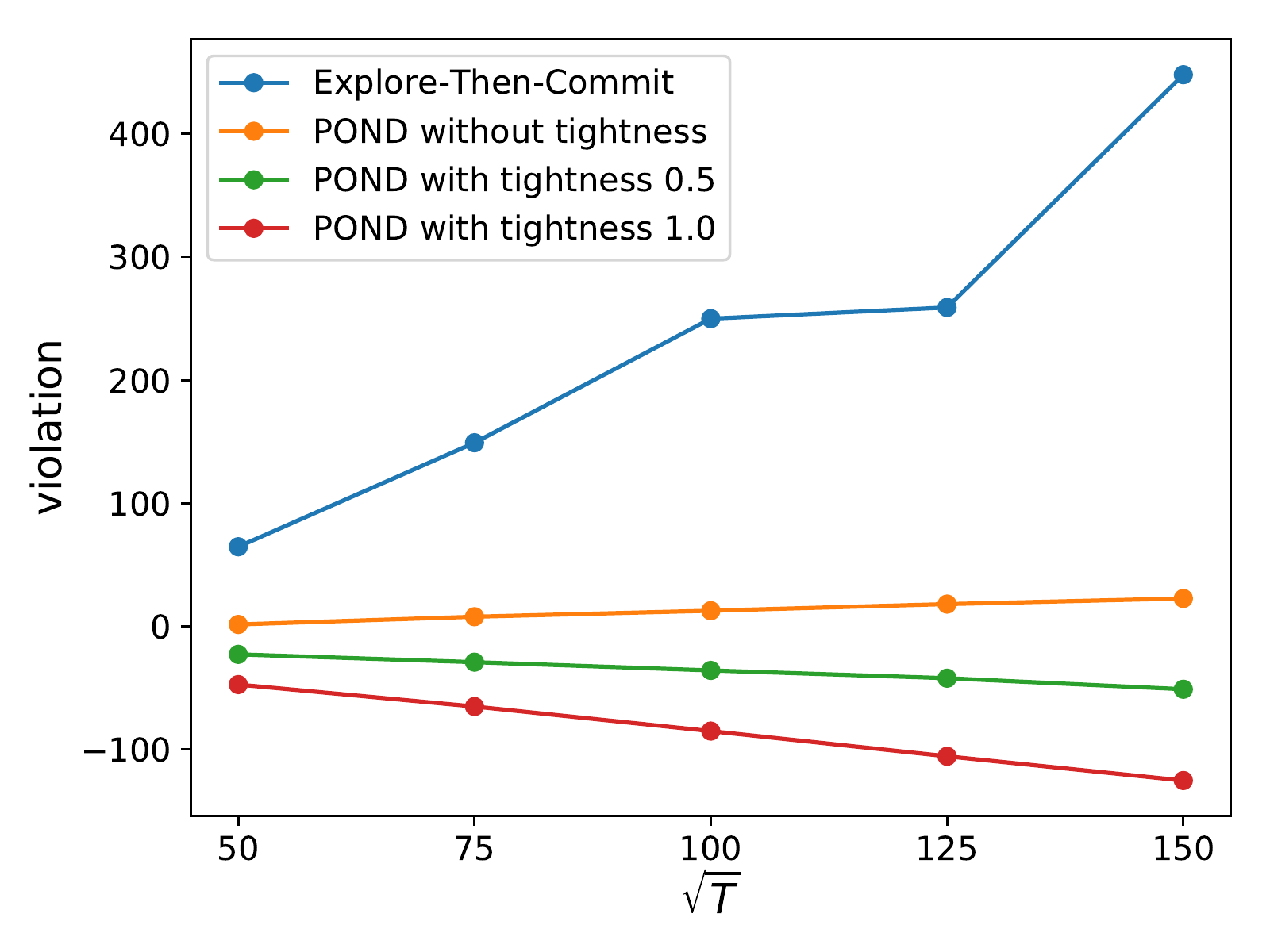}
  \caption{Resource violation}
   \label{fig:regret and cv v.s. time-rv}
\end{subfigure}
\caption{Regret and constraint violation v.s. $T$.}
\label{fig:regret and cv v.s. time}
\end{figure}

To further demonstrate the impact of the tightness constant on the constraint violations, we plotted the trajectories of regret and constraint violations for $T=10,000,$ with and without tightness, in Figures \ref{fig:trjectory with tightness} and \ref{fig:trjectory without tightness}, respectively. The shade regions in these figures include the trajectories of all trials and the thick curves within the shade regions are the corresponding average values. 
Figures \ref{fig:trjectory with tightness-cv}-\ref{fig:trjectory with tightness-rv} show that the violations (capacity and fairness violations) increase then decrease to negative. In contrast, the pink curve in Figure \ref{fig:trjectory without tightness-cv} shows that the capacity violation of server 1 increases to around 50, which is in the order of $O(\sqrt{T})$,  which corresponds to the middle point in the orange line in Figure \ref{fig:regret and cv v.s. time-cv}. The observations here further demonstrate the effectiveness of using tightness to reduce constraint violations.

\begin{figure}[h]
\centering
\begin{subfigure}{.24\textwidth}
  \centering
  \includegraphics[width=1.0\linewidth]{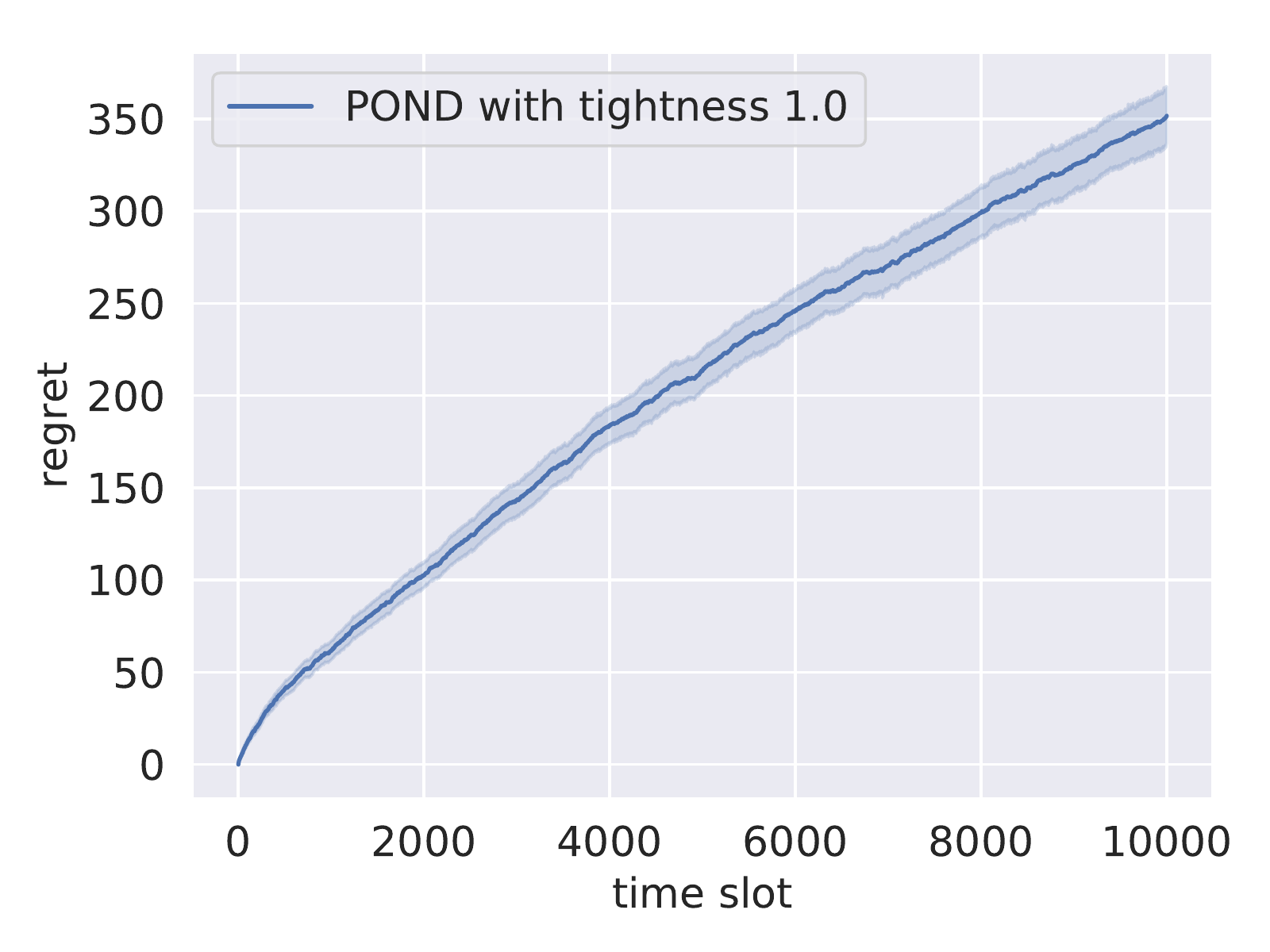}
  \caption{Regret}
  \label{fig:trjectory with tightness-regret}
\end{subfigure}%
\begin{subfigure}{.24\textwidth}
  \centering
  \includegraphics[width=1.0\linewidth]{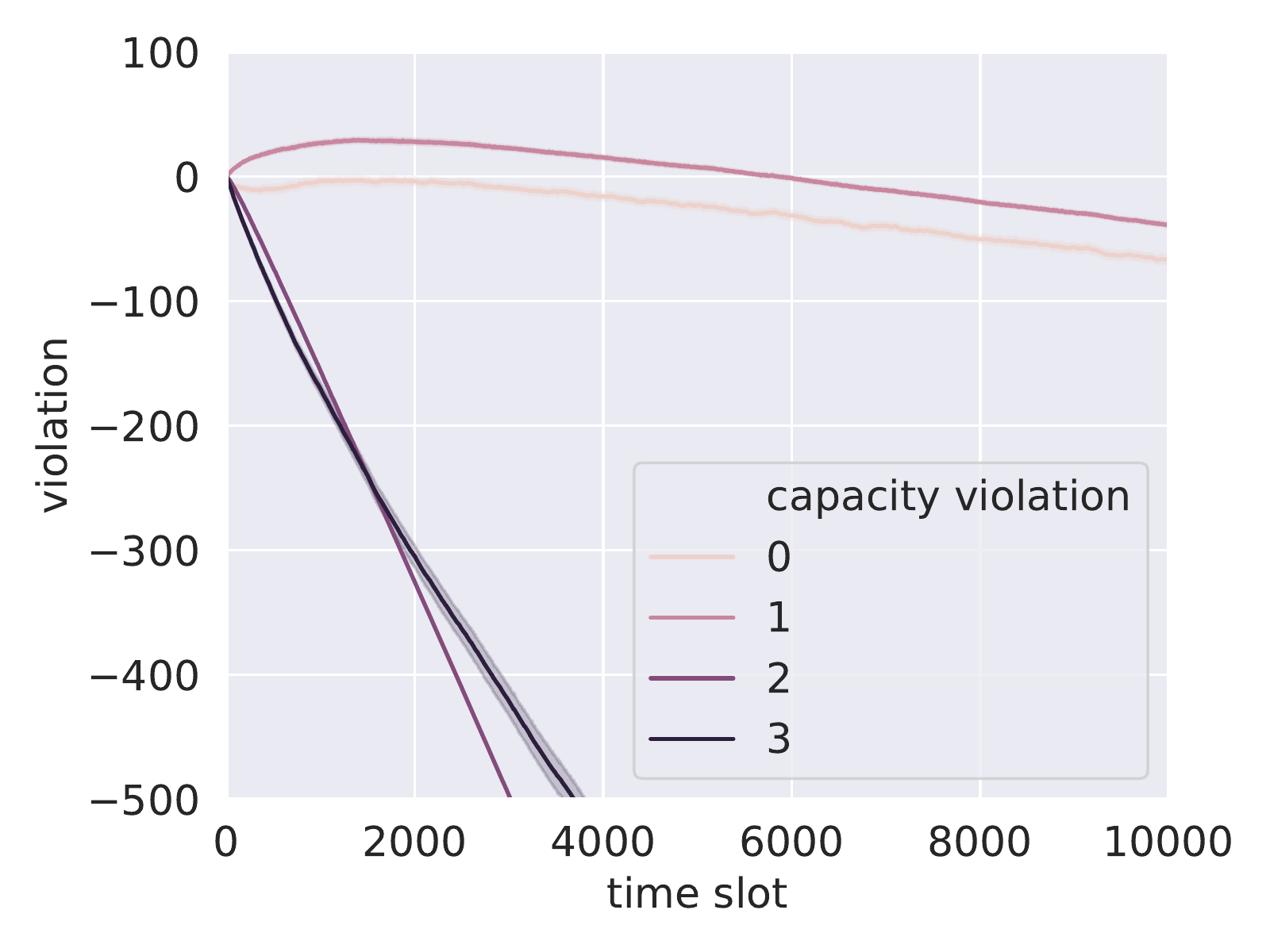}
  \caption{Capacity violation}
  \label{fig:trjectory with tightness-cv}
\end{subfigure}
\begin{subfigure}{.24\textwidth}
  \centering
  \includegraphics[width=1.0\linewidth]{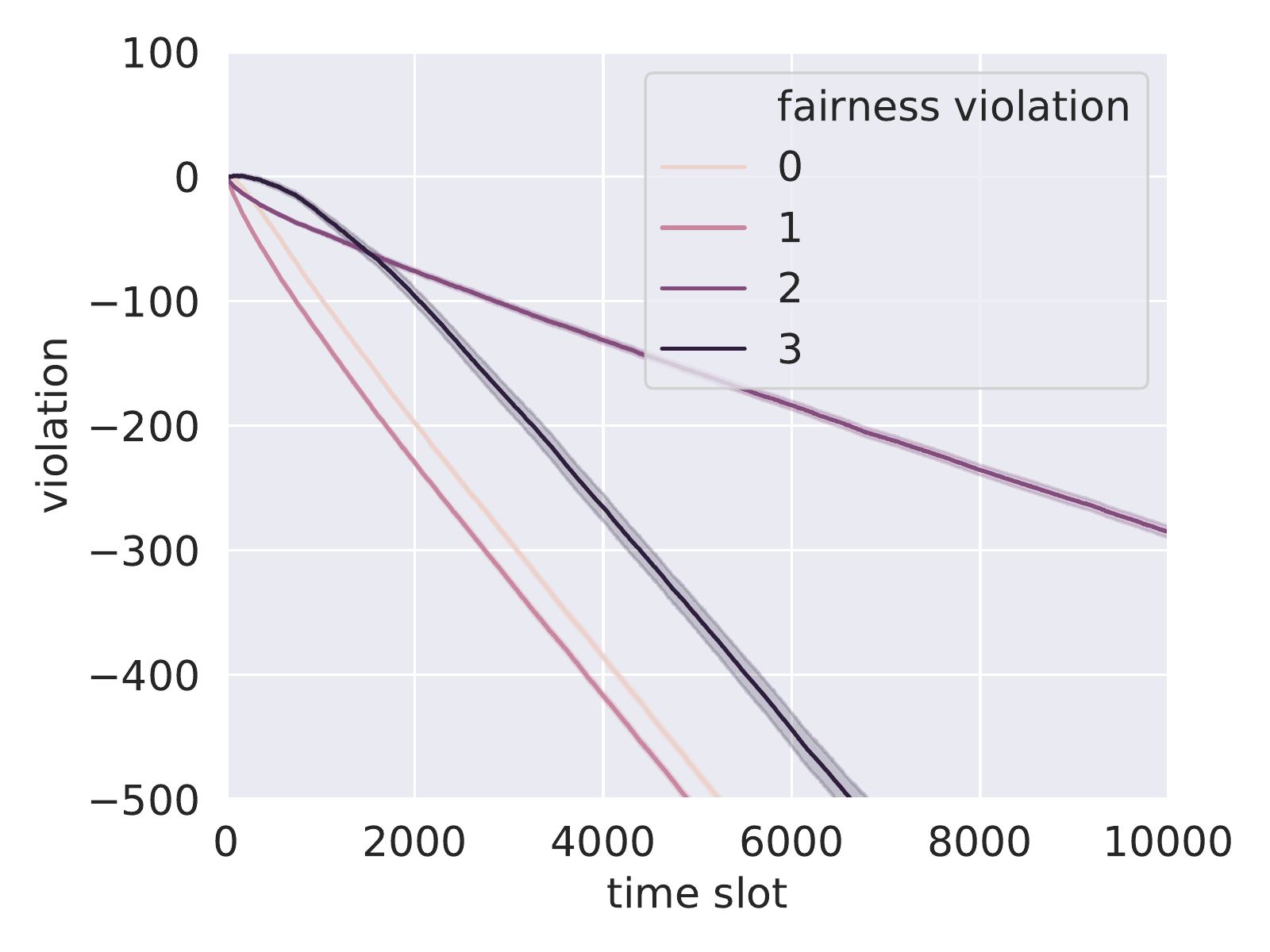}
  \caption{Fairness violation}
  \label{fig:trjectory with tightness-fv}
\end{subfigure}
\begin{subfigure}{.24\textwidth}
  \centering
  \includegraphics[width=1.0\linewidth]{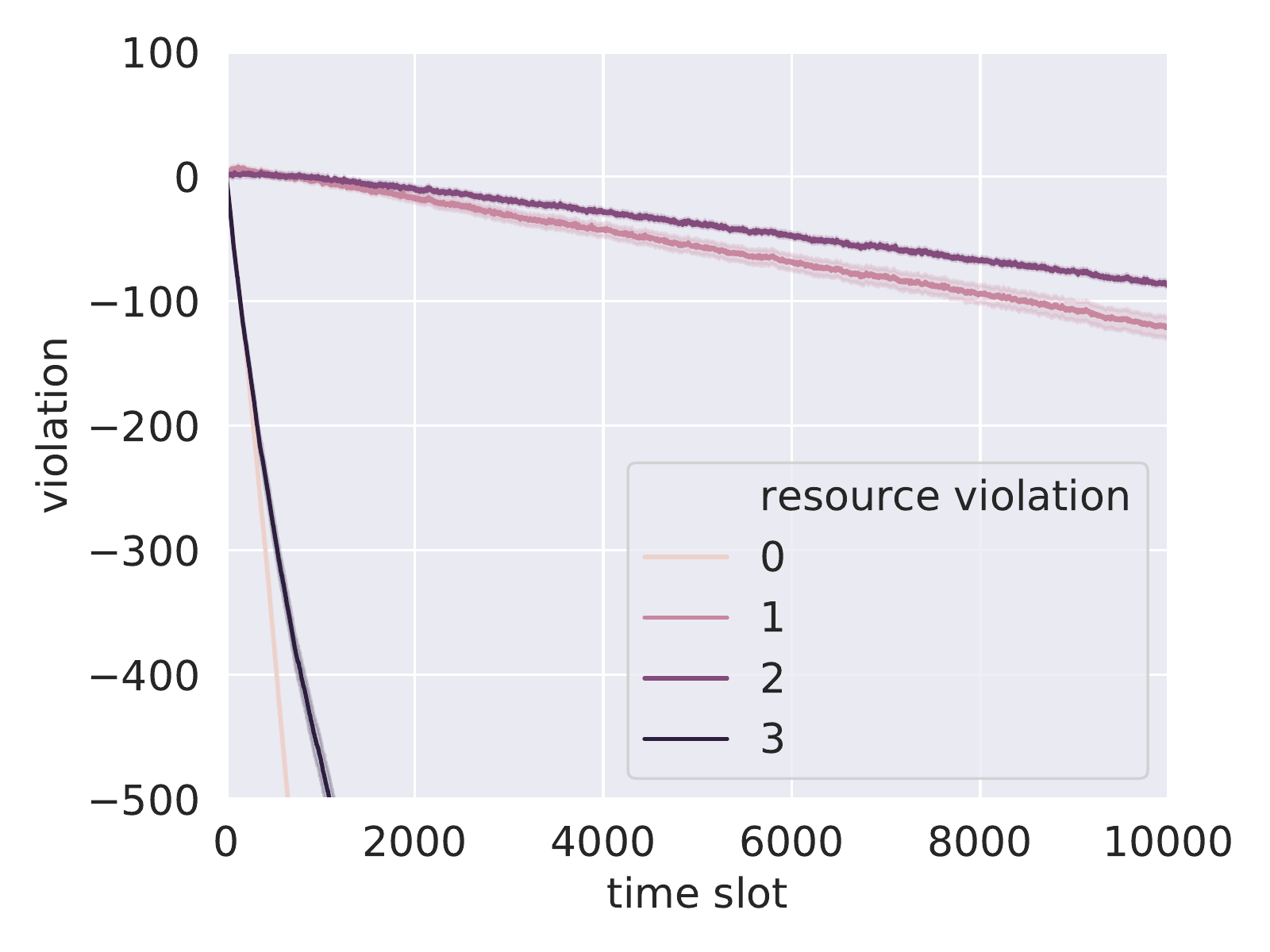}
  \caption{Resource violation}
  \label{fig:trjectory with tightness-rv}
\end{subfigure}
\caption{Trajectory of regret and constraint violation with tightness for $T=10000.$}
\label{fig:trjectory with tightness}
\end{figure} 

\begin{figure}[h]
\centering
\begin{subfigure}{.24\textwidth}
  \centering
  \includegraphics[width=1.0\linewidth]{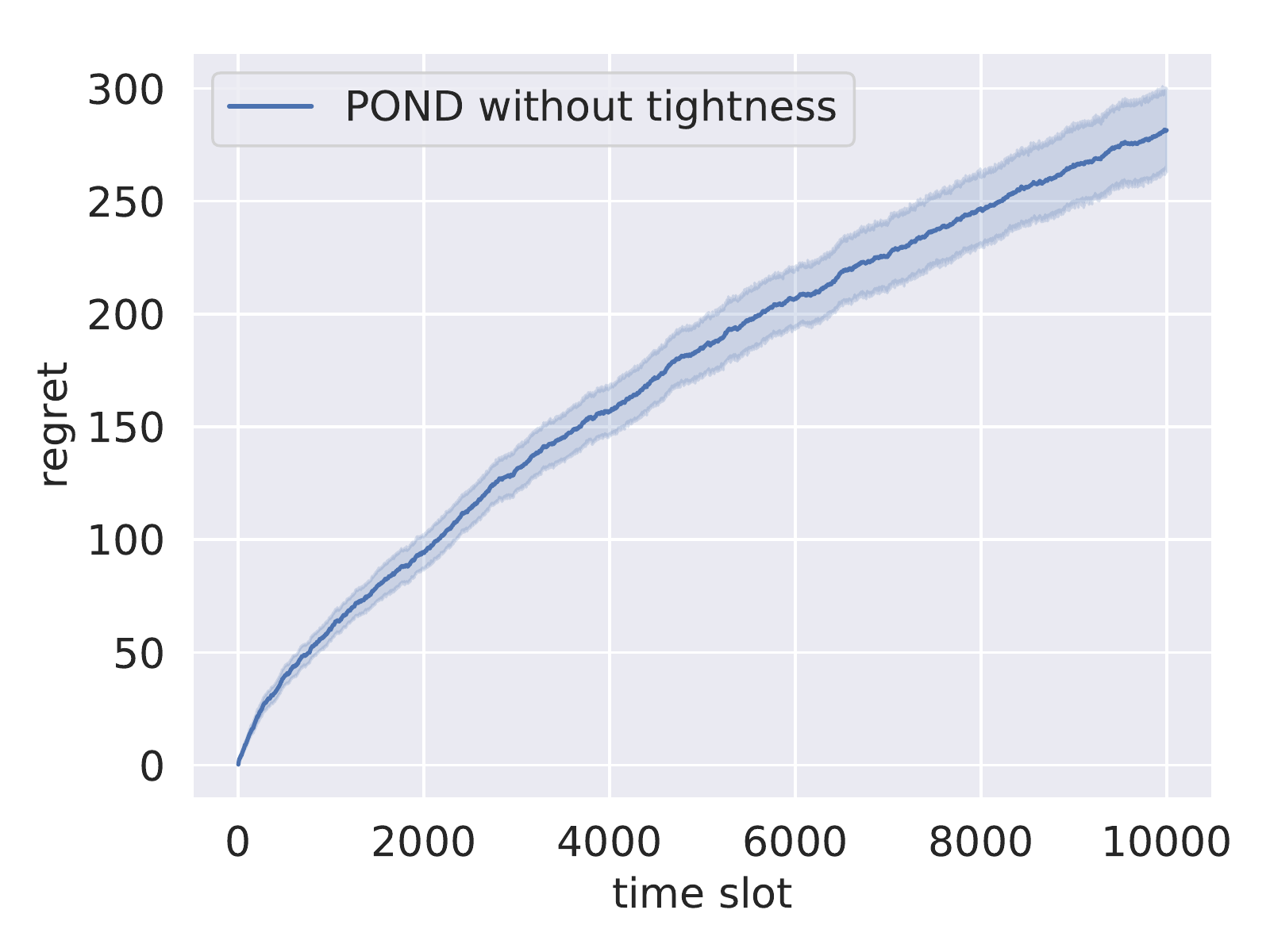}
  \caption{Regret}
  \label{fig:trjectory without tightness-regret}
\end{subfigure}%
\begin{subfigure}{.24\textwidth}
  \centering
  \includegraphics[width=1.0\linewidth]{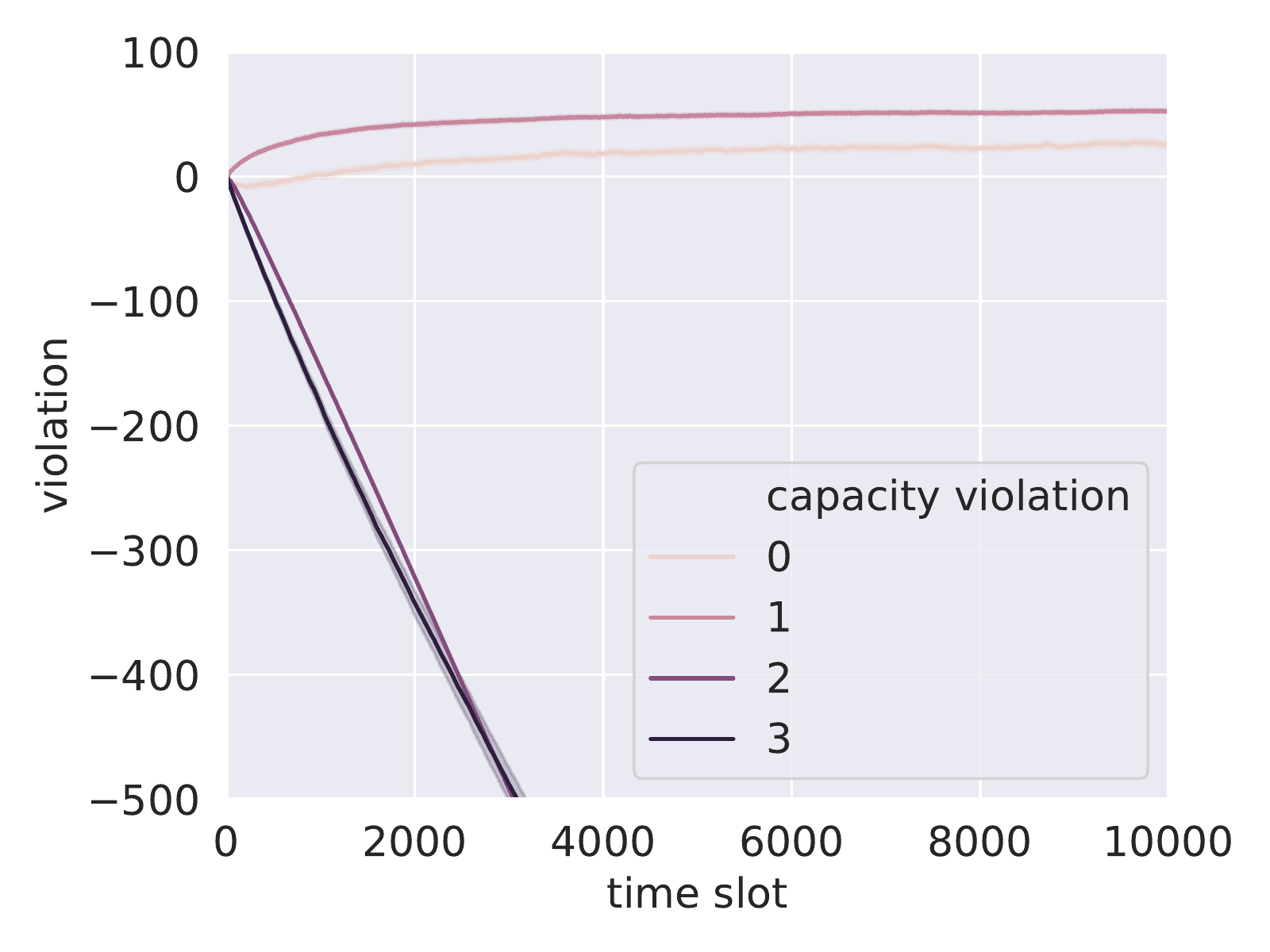}
  \caption{Capacity violation}
  \label{fig:trjectory without tightness-cv}
\end{subfigure}
\begin{subfigure}{.24\textwidth}
  \centering
  \includegraphics[width=1.0\linewidth]{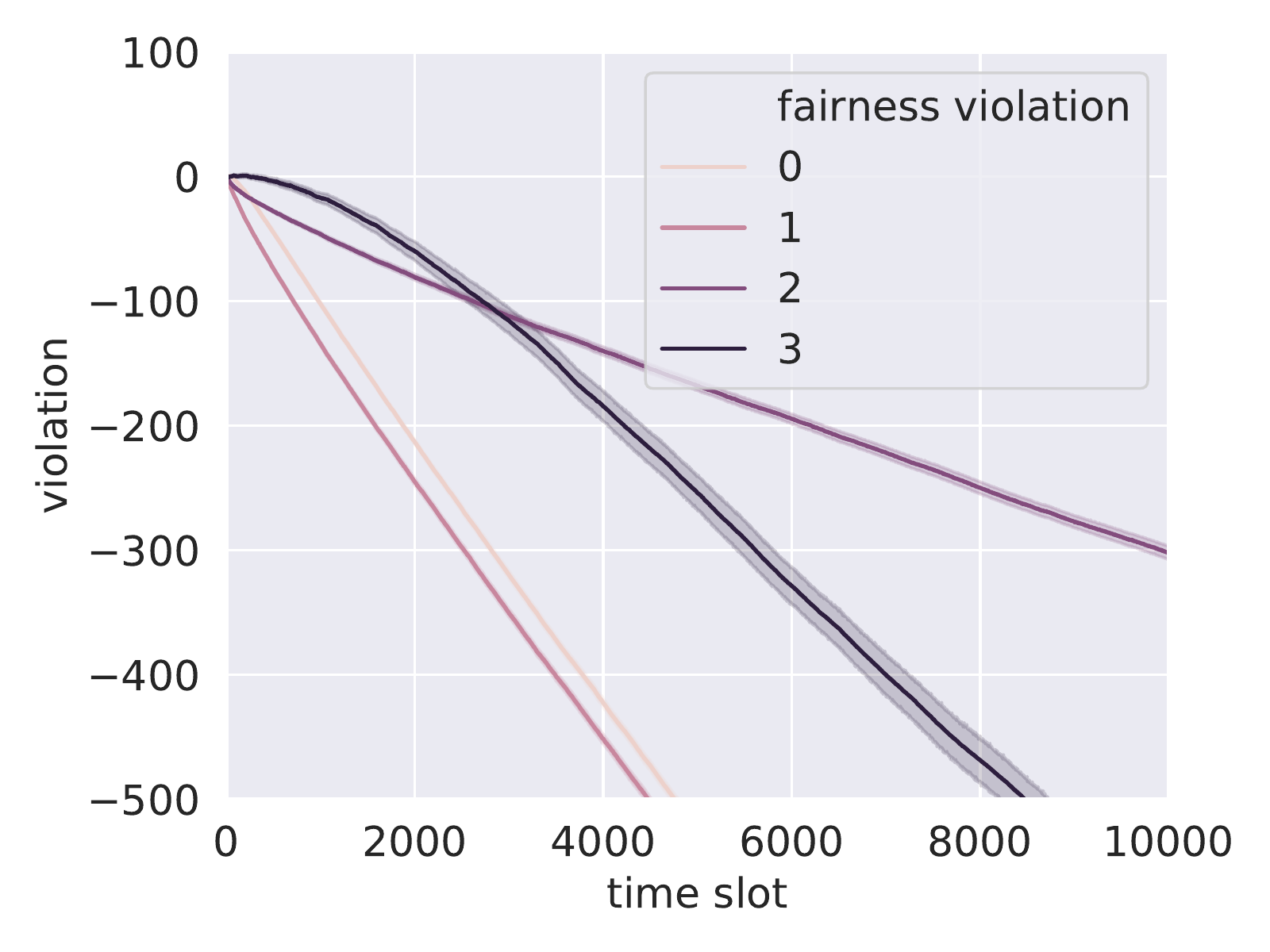}
  \caption{Fairness violation}
  \label{fig:trjectory without tightness-fv}
\end{subfigure}
\begin{subfigure}{.24\textwidth}
  \centering
  \includegraphics[width=1.0\linewidth]{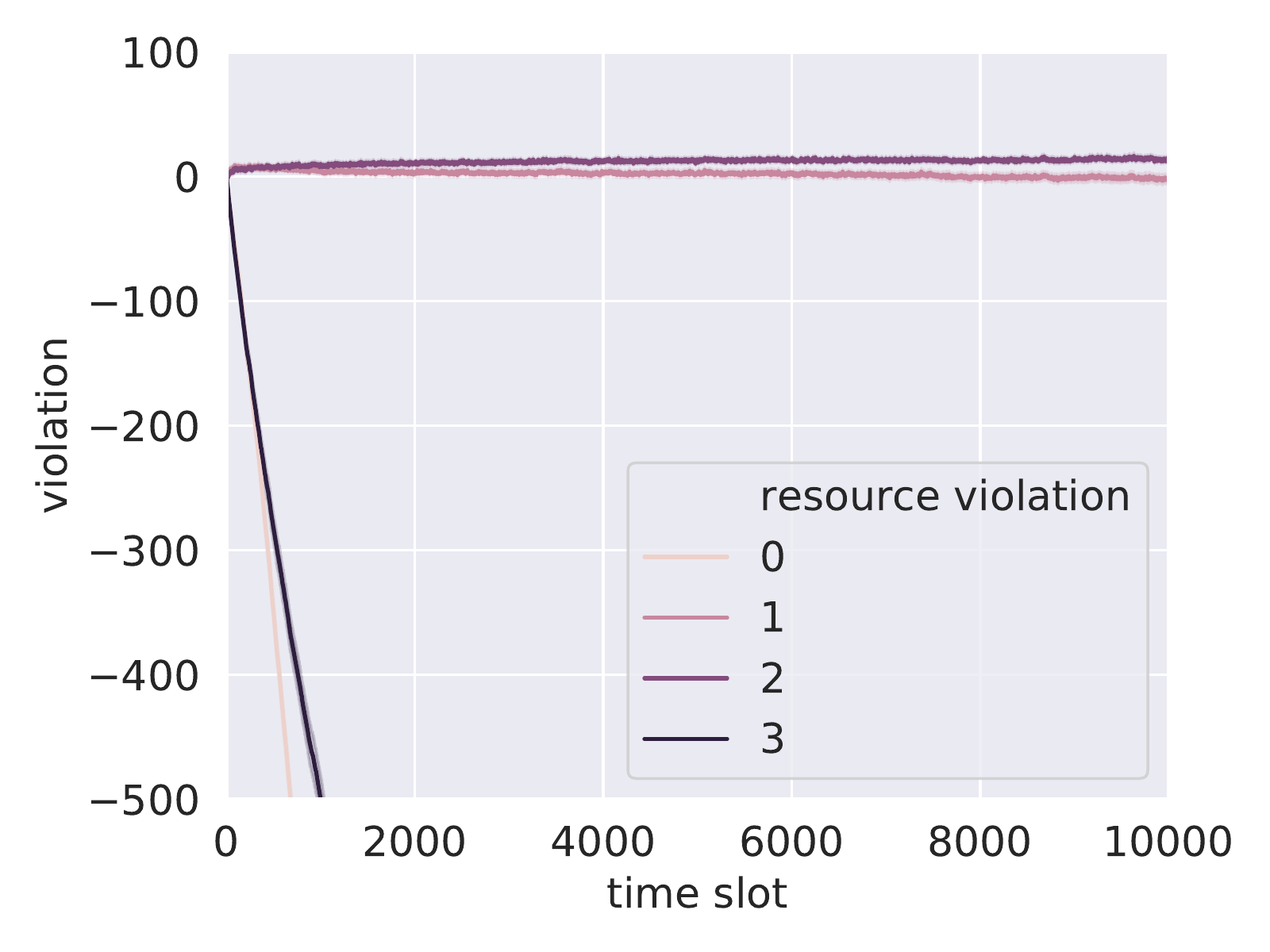}
  \caption{Resource violation}
  \label{fig:trjectory without tightness-rv}
\end{subfigure}
\caption{Regret and constraint violation without tightness for $T=10000.$}
\label{fig:trjectory without tightness}
\end{figure} 

\subsection{Online crowdsourcing: tutoring}
We performed experiments with the dataset on online tutoring in \cite{BloSteSar_19}, where the data was collected from the crowdsourcing marketplace
Amazon Mechanical Turk. In the dataset, users classified with gender (male or female) were presented with three different tutorial types. Users after tutoring were assessed with ten related questions and assessment scores (range from $0$ to $10$ and we normalized it into $[0,1]$) were collected. The total number of valid data points is $2,582.$ Since the dataset was collected by an algorithm that dispatched jobs uniformly at random, we performed the reject sampling method to evaluate the algorithms as in \cite{LiChuLan_10}.
We imposed the following constraints: capacity constraints are $\sum_{i=1}^2x_{i,j} \leq \mu_j$ with $\boldsymbol \mu = [1/3, 0.4, 1/3];$
fairness constraints are $\sum_{i=1}^2 x_{i,j} \geq d_j\sum_{i=1}^2 \lambda_i$ with $\mathbf d = [0.3, 0.3, 0.3];$
resource constraints are $\sum_{i=1}^2w_{i,j }x_{i,j} \leq \rho_j$ with 
$\mathbf w = \begin{bmatrix}
1 & 1 & 1.5\\
1.5 & 1 & 1
\end{bmatrix}$
and $\boldsymbol \rho = [1/2, 0.35, 1/3].$
We performed $100$ trails for POND and Explore-Then-Commit, where the bootstrapping method was utilized to run for $10,000$ time slots. The results with POND with tightness $1.0$ and Explore-Then-Commit are shown in Figure \ref{fig: reward POND and ETC}, \ref{fig: turoting POND} and \ref{fig: turoting ETC}. The average accumulated rewards ($\text{accumulated rewards}/T$) are shown in Figure \ref{fig: reward POND and ETC}, where POND achieved larger average reward ($0.366$) than Explore-Then-Commit ($0.355$) in $T=10,000$ time slots. Moreover, POND led to much lower constraint violations in Figure \ref{fig: turoting POND tightness-cv}-\ref{fig: turoting POND tightness-rv} than Explore-Then-Commit in Figure \ref{fig: turoting ETC-cv}-\ref{fig: turoting ETC-rv}.

\begin{figure}[h]
\centering
  \includegraphics[width=2.5in]{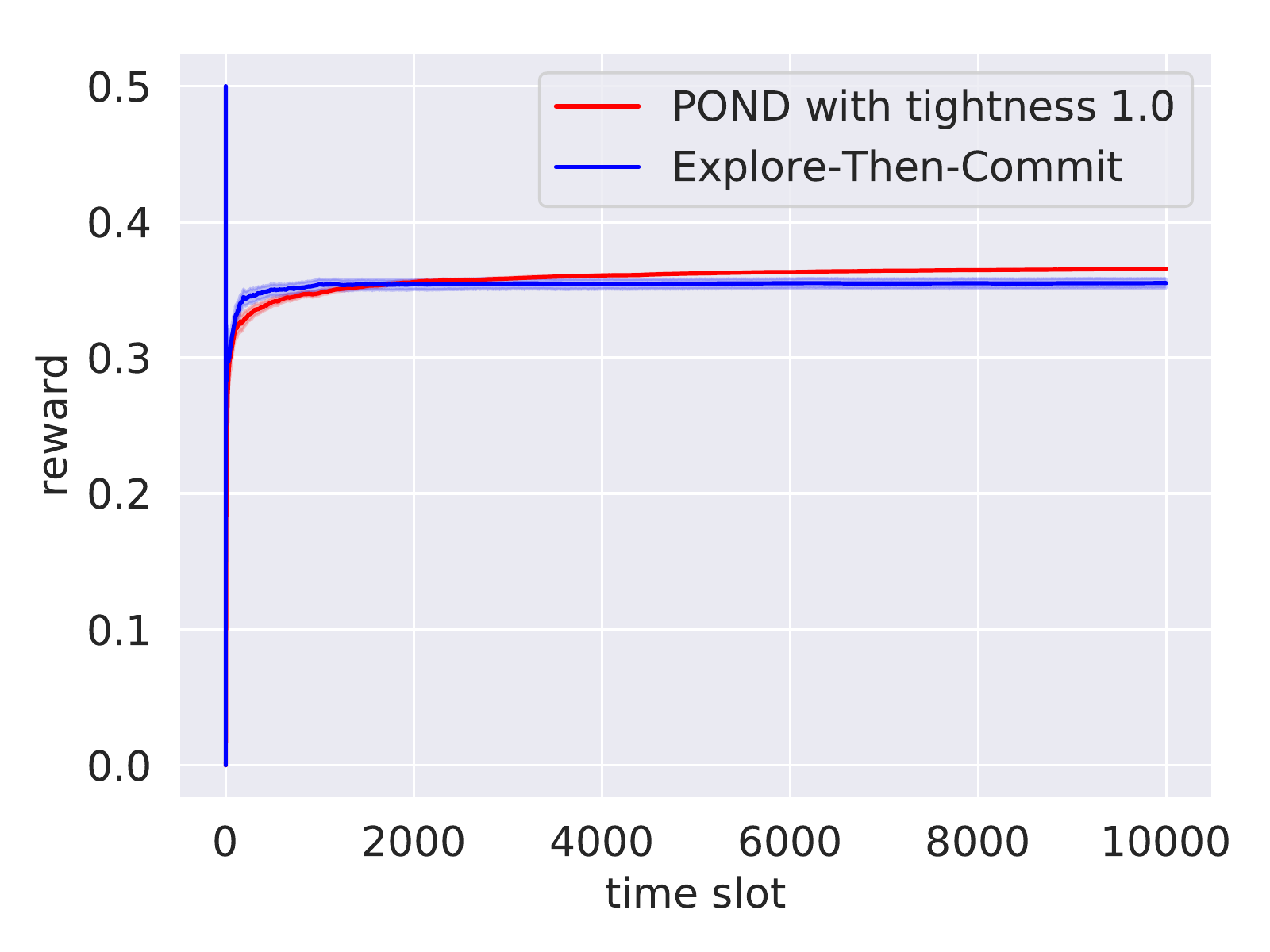}
  \caption{Reward: POND $(0.366)$ v.s. Explore-Then-Commit $(0.355).$}
  \label{fig: reward POND and ETC}
\end{figure}

\begin{figure}[h]
\centering
\begin{subfigure}{.32\textwidth}
  \centering
  \includegraphics[width=1.0\linewidth]{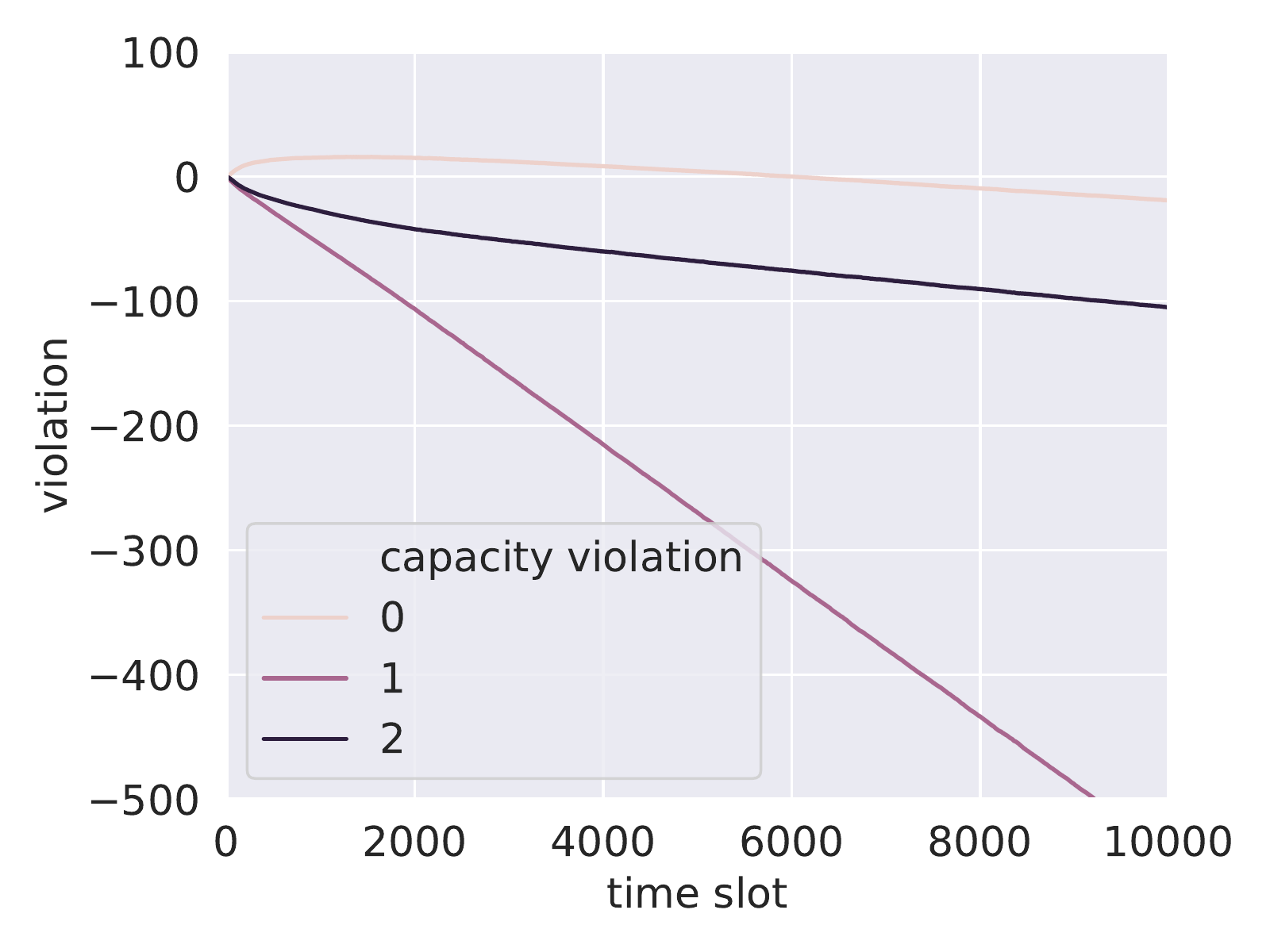}
  \caption{Capacity violation}
  \label{fig: turoting POND tightness-cv}
\end{subfigure}
\begin{subfigure}{.32\textwidth}
  \centering
  \includegraphics[width=1.0\linewidth]{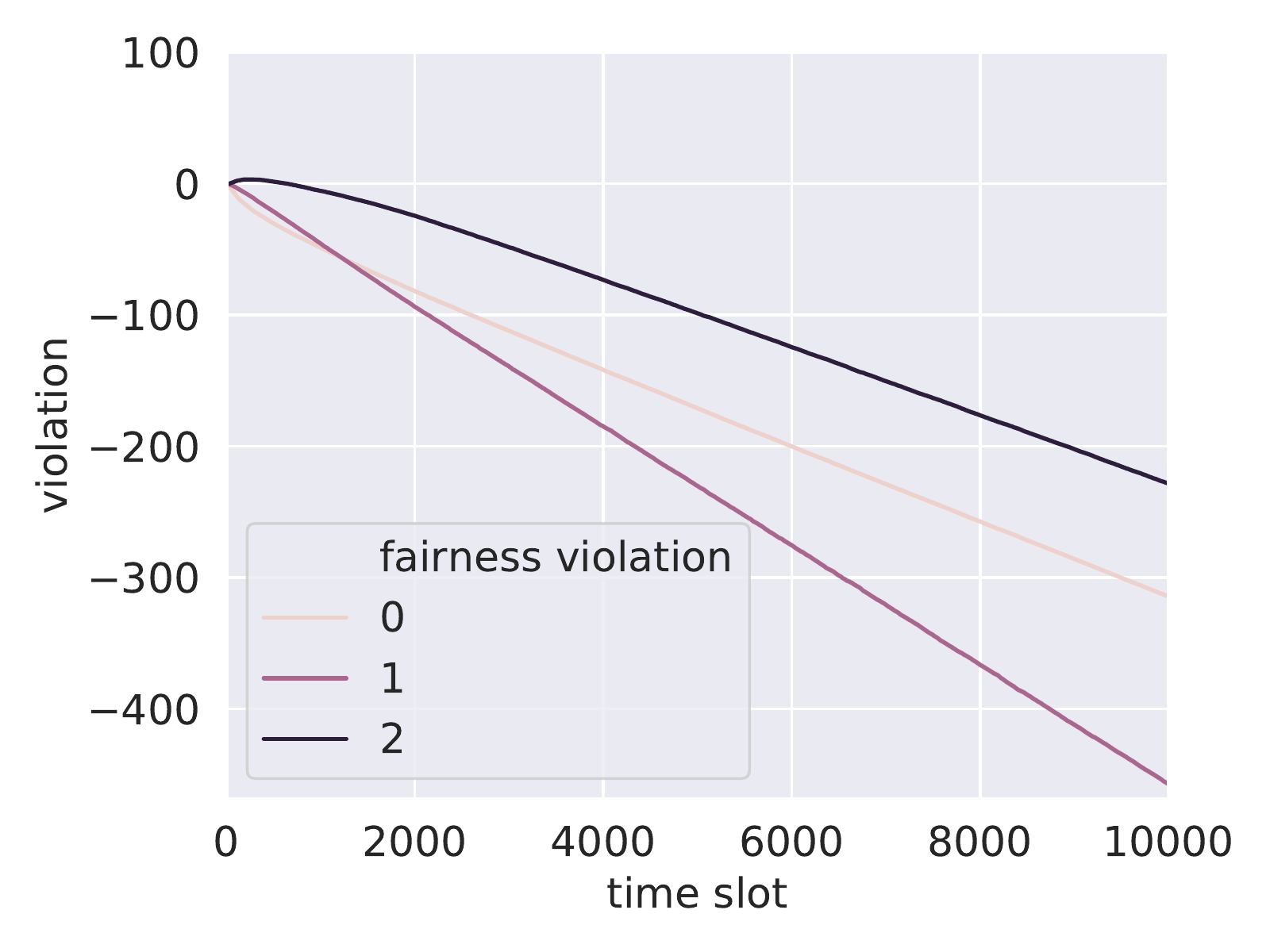}
  \caption{Fairness violation}
  \label{fig: turoting POND tightness-fv}
\end{subfigure}
\begin{subfigure}{.32\textwidth}
  \centering
  \includegraphics[width=1.0\linewidth]{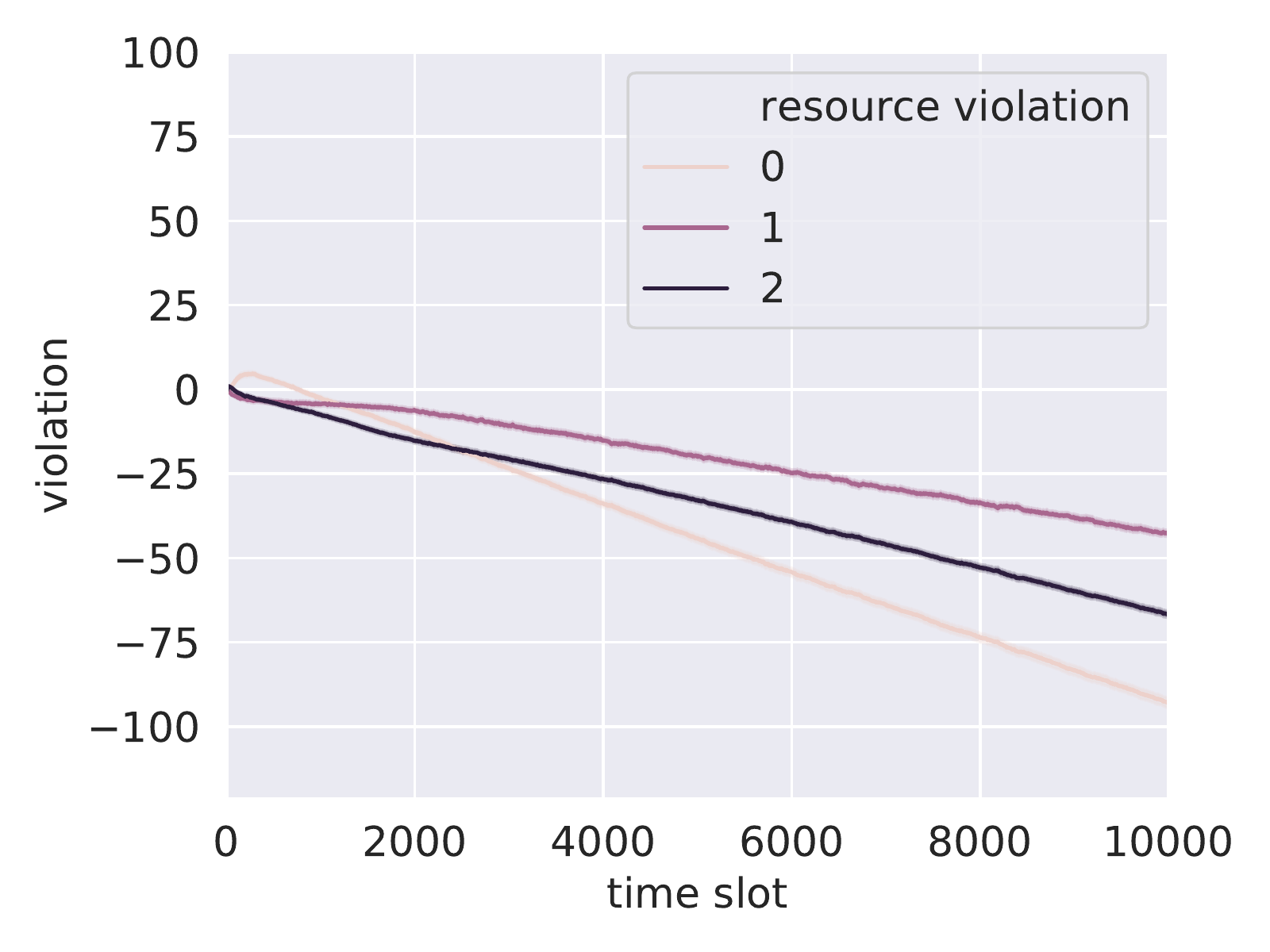}
  \caption{Resource violation}
  \label{fig: turoting POND tightness-rv}
\end{subfigure}
\caption{Constraint violations under POND tightness $1.0$.}
\label{fig: turoting POND}
\end{figure} 

\begin{figure}[h]
\centering
\begin{subfigure}{.32\textwidth}
  \centering
  \includegraphics[width=1.0\linewidth]{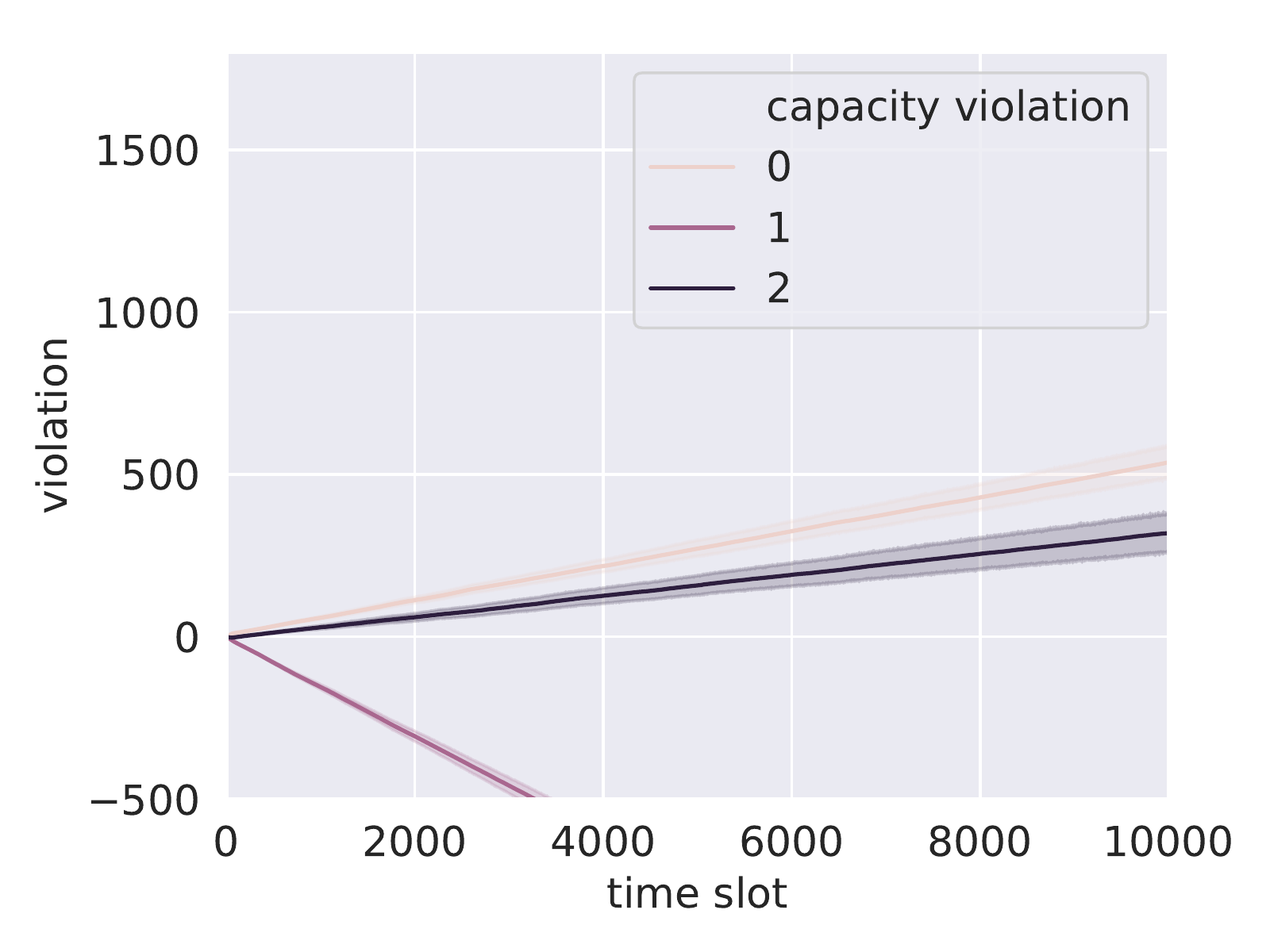}
  \caption{Capacity violation}
  \label{fig: turoting ETC-cv}
\end{subfigure}
\begin{subfigure}{.32\textwidth}
  \centering
  \includegraphics[width=1.0\linewidth]{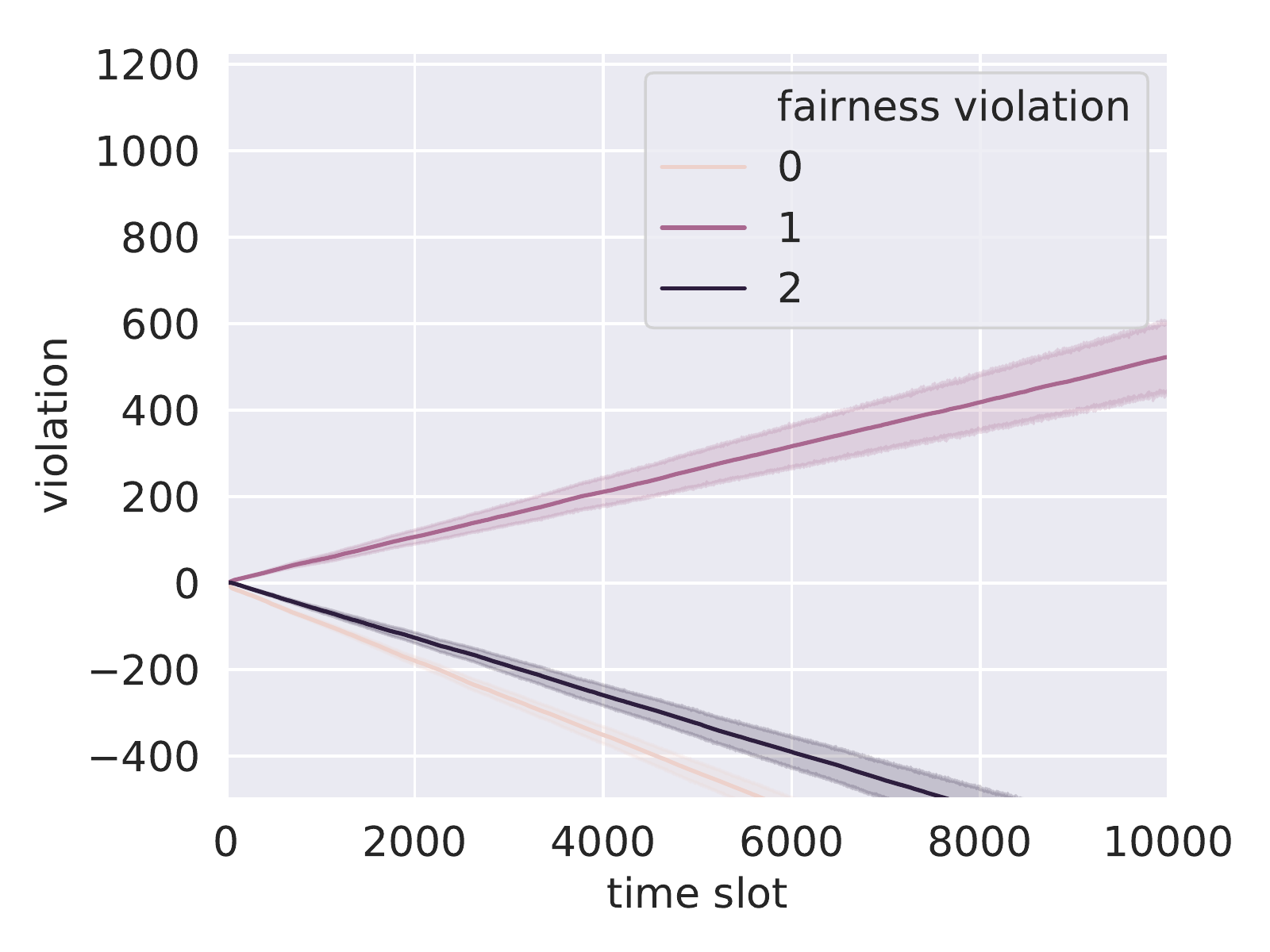}
  \caption{Fairness violation}
  \label{fig: turoting ETC-fv}
\end{subfigure}
\begin{subfigure}{.32\textwidth}
  \centering
  \includegraphics[width=1.0\linewidth]{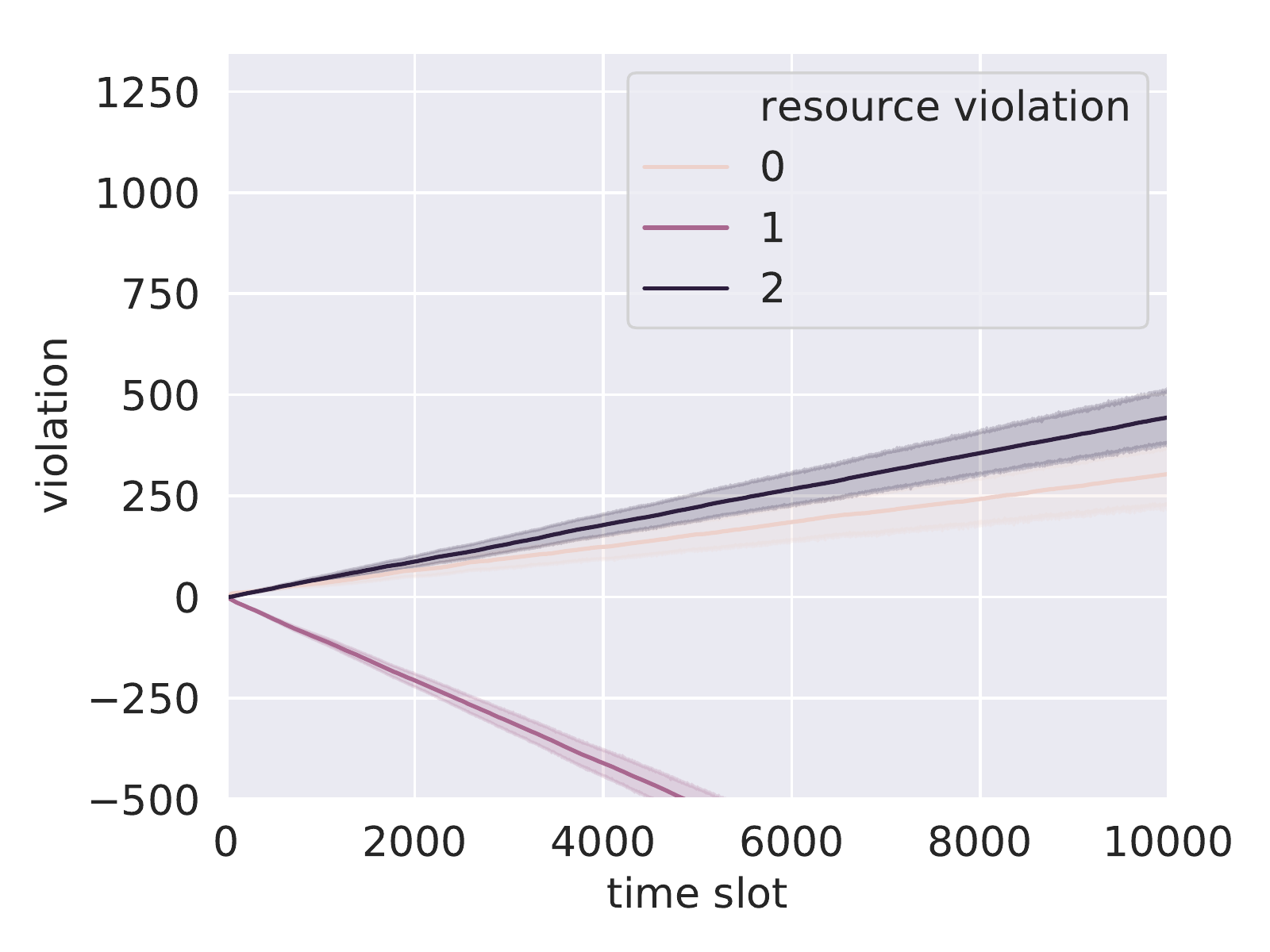}
  \caption{Resource violation}
  \label{fig: turoting ETC-rv}
\end{subfigure}
\caption{Constraint violations under Explore-Then-Commit.}
\label{fig: turoting ETC}
\end{figure} 

\section{Conclusion}

In this paper, we developed a novel online dispatching algorithm, POND, to maximize cumulative reward over a finite time horizon, subject to general constraints that arise from resource capacity and fairness considerations. Given unknown arrival,  reward, and constraint distributions, POND leverages the UCB approach to learn the reward while using the MaxWeight algorithm to make the dispatching decision with virtual queues tracking the constraint violations. We used Lyapunov drift analysis and regret analysis to show that our POND algorithm can achieve $O(\sqrt{T})$ regret while a constant $O(1)$ constraint violation, with the key being introducing a ``tightness'' constant to balance between the regret and constraint violation. Via numerical experiments based on synthetic data and a real dataset, we further demonstrated the critical role of this tightness constant in achieving the $O(1)$ constraint violation, showing our bounds are sharp. We also showed that our algorithm performs significantly better than the baseline Explore-Then-Commit algorithm, leading to both smaller regret and constraint violation in experiments with the synthetic data and real data.

\newpage
\bibliographystyle{ACM-Reference-Format}
\bibliography{inlab-ref}

\appendix

\section{Proof of Lemma \ref{LEMMA: DRIFT-INEQUALITY}}\label{app: lemma-drift-inequality}

\lemmadriftinequality*
\begin{proof} Since $\mathbf x_{\epsilon}$ is a feasible solution in $\mathcal X_{\epsilon},$ we have $\lambda_i = \sum_{j} x_{\epsilon,i,j}.$ Recall $\eta_{i,j} = \hat r_{i,j} - \sum_k Q_j^{(k)} w_{i,j}^{(k)}$ in Algorithm \ref{main alg} and let $j^* \in \argmax_j \eta_{i,j}.$ 
Given $\mathbf H(t)=\mathbf h,$ we have
\begin{align}
&V \sum_{i}\sum_{j} \hat r_{i,j} x_{\epsilon,i,j}-\sum_j\sum_{k}Q_j^{(k)}\left(\sum_{i}w_{i,j}^{(k)} x_{\epsilon,i,j} - \rho_j^{(k)} + \epsilon\right) \nonumber \\
=& \sum_{i}\sum_{j} \left(V\hat r_{i,j} x_{\epsilon,i,j}-\sum_{k}Q_j^{(k)}w_{i,j}^{(k)} x_{\epsilon,i,j}\right) - \sum_{j}\sum_{k} \left(- \rho_j^{(k)} + \epsilon\right) \nonumber \\
\leq& \sum_{i} \left(V\hat r_{i,j^*} \lambda_i-\sum_{k}Q_{j^*}^{(k)}w_{i,j^*}^{(k)} \lambda_i\right) - \sum_{j}\sum_{k} \left(- \rho_j^{(k)} + \epsilon\right) \nonumber \\
=& \mathbb E\left[  \sum_{i} V\hat r_{i,j^*}(t) \Lambda_{i}(t)-\sum_{k}Q_{j^*}^{(k)}w_{i,j^*}^{(k)}(t) \Lambda_{i}(t) \Big| \mathbf H(t) = \mathbf h \right] - \sum_{j}\sum_{k} \left(- \rho_j^{(k)} + \epsilon\right) \nonumber \\
=& \mathbb E\left[  \sum_{i} \sum_{j} V \hat r_{i,j}(t) x_{i,j}(t)-\sum_{j}\sum_{k}Q_j^{(k)}\left(\sum_iw_{i,j}^{(k)}(t) x_{i,j}(t) - \rho_j^{(k)}(t) + \epsilon\right) \Big| \mathbf H(t) = \mathbf h \right] \nonumber 
\end{align}
where the first inequality holds because $\lambda_{i} = \sum_j  x_{\epsilon,i,j}$ and $j^* \in \argmax_j \eta_{i,j};$ the last equality holds because $x_{i,j}(t)$ maximizes $\sum_{i, j} \left(V \hat r_{i,j}(t) - \sum_k w_{i,j}^{(k)}(t)Q_j^{(k)}(t)\right) x_{i,j}$ with $\Lambda_{i}(t) = \sum_j  x_{i,j}(t)$ in Algorithm \ref{main alg}.
\end{proof}

\section{Proof of Lemma \ref{lemma: elsilon gap}}

\lemmaelsilongap*
\begin{proof}
Since $\mathbf x^*$ is the optimal solution to optimization problem \eqref{obj-fluid}-\eqref{resource limit-fluid}, we have 
\begin{align*}
&\lambda_i = \sum_{j}  x_{i,j}^*, \forall i \in \mathcal N, x_{i,j}^* \geq 0, \forall i \in \mathcal N, j \in \mathcal M,\\
&\sum_{i}  w_{i,j}^{(k)} x_{i,j}^* \leq \rho_j^{(k)}, ~\forall j \in \mathcal M, \forall k \in \mathcal K. 
\end{align*}
Under Assumption \ref{assumption:slater}, there exists $\mathbf x^{\text{in}} \in \mathcal X$ such that \begin{align*}
&\lambda_i = \sum_{j}  x_{i,j}^{\text{in}}, \forall i \in \mathcal N, x_{i,j}^{\text{in}} \geq 0, \forall i \in \mathcal N, j \in \mathcal M,\\
&\sum_{i}  w_{i,j}^{(k)} x_{i,j}^{\text{in}} \leq \rho_j^{(k)} -\delta, ~\forall j \in \mathcal M, \forall k \in \mathcal K. 
\end{align*}
We construct $\mathbf x_\epsilon = \left(1-\frac{\epsilon}{\delta}\right) \mathbf x^* + \frac{\epsilon}{\delta} \mathbf x^{\text{in}}$ such that it satisfies that
\begin{align*}
&\sum_{j} x_{\epsilon, i,j} = \left(1-\frac{\epsilon}{\delta}\right) \sum_{j} x_{i,j}^* + \frac{\epsilon}{\delta}  \sum_{j} x_{i,j}^{\text{in}} = \lambda_i, \forall i \in \mathcal N
\end{align*}
and 
\begin{align*}
\sum_{i} w_{i,j}^{(k)} x_{\epsilon, i,j} = \sum_{i} w_{i,j}^{(k)} \left[\left(1-\frac{\epsilon}{\delta}\right)x_{i,j}^* +  \frac{\epsilon}{\delta} x_{i,j}^{\text{in}} \right] \leq \rho_j^{(k)} -\epsilon, ~\forall j \in \mathcal M, \forall k \in \mathcal K. 
\end{align*}
Therefore, $\mathbf x_\epsilon$ is a feasible solution to ``$\epsilon$-tight'' optimization problem \eqref{obj-fluid-tight} - \eqref{resource limit-fluid-tight} that
\begin{align*}
    \max_{\mathbf x} & ~  \sum_{i,j} r_{i,j} x_{i,j}\\
    \text{s.t.} & ~ \lambda_i = \sum_{j}  x_{i,j}, ~\forall i \in \mathcal N,  ~x_{i,j} \geq 0, \forall i \in \mathcal N, j \in \mathcal M,\\
    & ~ \sum_{i}  w_{i,j}^{(k)} x_{i,j} + \epsilon \leq \rho_j^{(k)}, ~\forall j \in \mathcal M, \forall k \in \mathcal K. 
\end{align*}
Given $\mathbf x_\epsilon^*$ is an optimal solution to ``$\epsilon$-tight'' optimization problem above,
we have
\begin{align*} 
 \sum_{i,j} r_{i,j} (x^*_{i,j} -x^*_{\epsilon,i,j} )
\leq & \sum_{i,j} r_{i,j} (x^*_{i,j} -x_{\epsilon,i,j}) \\
= & \sum_{i,j} r_{i,j} \left(x^*_{i,j} - \left(1-\frac{\epsilon}{\delta}\right)  x^*_{i,j} - \frac{\epsilon}{\delta} x_{i,j}^{\text{in}}\right) \\
\leq & \sum_{i,j} r_{i,j} \left(x^*_{i,j} - \left(1-\frac{\epsilon}{\delta}\right)  x^*_{i,j} \right) \\
\leq & \sum_{i,j} \left(x^*_{i,j} - \left(1-\frac{\epsilon}{\delta}\right)  x^*_{i,j} \right) \\
= & \frac{\epsilon}{\delta}\sum_{i,j} x^*_{i,j} = \frac{\epsilon}{\delta} \sum_i \lambda_i
\end{align*}
where the first inequality holds because $\mathbf x_\epsilon^*$ is the optimal solution and $\mathbf x_\epsilon$ is a feasible solution; the first equality holds because $\mathbf x_\epsilon = \left(1-\frac{\epsilon}{\delta}\right) \mathbf x^* + \frac{\epsilon}{\delta} \mathbf x^{\text{in}};$ the third inequality holds because $x_{i,j}^{\text{in}}$ is non-negative.

\end{proof}

\section{Proof of Lemma  \ref{MAB-negative-lemma}}

\subsection{MOSS learning}
Let $\bar R_{i,j,s} = \frac{1}{s}\sum_{k=1}^sR_{i,j,k}$ be the empirical mean based on $s$ sampled rewards of a type $i$ job assigned to server $j.$ As a standard trick used in regret analysis, we assume the sampled rewards, e.g. $R_{i,j,k}$ is drawn before the system starts, but its value is revealed only when the $k$th type-$i$ job is served by server $j.$ In this way, the distribution of $\bar R_{i,j,s}$ is independent of the online dispatching decisions, so we can apply the Hoeffding inequality. 
Next, define $\zeta_{i,j} = r_{i,j} - \min_{s \leq C_{\lambda}T} \left( \bar R_{i,j,s} + \sqrt{\frac{2}{s}\log \frac{T}{M s}}\right).$ We have the inequality to be $$\mathbb P(r_{i,j} - \hat{r}_{i,j}(t) > y) \leq \mathbb P(\zeta_{i,j} > y).$$  We then introduce a tight version of Hoeffding inequality to study $\mathbb P(\zeta_{i,j} > y)$ in the following lemma \cite{Hoe_63}.
\begin{lemma}\label{Hoeffding}
Let $X_1, X_2, \cdots, X_s$ to be the i.i.d random variables with $X_i \in [0, 1]$ and $\mathbb E[X_i] = \mu,$ we have
$$\mathbb P\left(\exists 1\leq s \leq m, ~ \sum_{i=1}^s (\mu - X_i) > y\right) \leq e^{-\frac{2y^2}{m}}.$$
\end{lemma}
Based on Lemma \ref{Hoeffding}, we first study the tail probability:
\begin{align*}
    \mathbb P(\zeta_{i,j} > y) =& \mathbb P\left(r_{i,j} - \min_{s \leq C_{\lambda}T} \left( \bar R_{i,j,s} + \sqrt{\frac{2}{s}\log \frac{T}{M s}}\right) > y\right) \\
    =& \mathbb P\left(\exists s \leq C_{\lambda}T, ~r_{i,j} - \bar R_{i,j,s} - \sqrt{\frac{2}{s}\log \frac{T}{M s}} > y\right) \\
    =& \mathbb P\left(\exists s \leq C_{\lambda}T, ~s r_{i,j} - s \bar R_{i,j,s} > sy + \sqrt{2s\log \frac{T}{M s}}\right) \\
    \leq& \sum_{j=0}^{\infty} \mathbb P\left(\exists s, ~ 2^j \leq s \leq 2^{j+1}, ~s r_{i,j} - s \bar R_{i,j,s} > 2^jy + \sqrt{2^{j+1}\log \frac{T}{M2^{j+1}}}\right)
\end{align*}
By invoking Lemma \ref{Hoeffding}, it implies
\begin{align*}
    \mathbb P(\zeta_{i,j} > y)
    \leq \sum_{j=0}^{\infty} e^{-\frac{\left(2^jy + \sqrt{2^{j+1}\log \frac{T}{M2^{j+1}}}\right)^2}{2^{j}}}
    \leq \sum_{j=0}^{\infty} e^{-2^{j}y^2-2\log \frac{T}{M2^{j+1}}}
    = \frac{M^2}{T^2}\sum_{j=0}^{\infty} 4^{j+1} e^{-2^{j}y^2}. 
\end{align*}
where the second inequality holds because $(a+b)^2 \geq a^2+b^2, \forall a, b\geq 0.$ Then we have 
\begin{align*}
\frac{M^2}{T^2}\sum_{j=0}^{\infty} 4^{j+1} e^{-2^{j}y^2} 
\leq \frac{16M^2}{e^2T^2y^2} + \int_0^{\infty} 4^{j+1} e^{-2^{j}y^2} dj \leq \left(\frac{9}{y^2}+\frac{6}{y^4}\right)\frac{M^2}{T^2},
\end{align*}
where we use $\sum_{j=a}^b f(j) \leq \max_{a \leq s \leq b} f(s) + \int_{b}^a f(s)ds$ for any function $f(s)$ is unimodal in $[a,b].$ Therefore, we have
\begin{align}
    \mathbb E[\zeta_{i,j}] 
    = \int_0^{\infty}\mathbb P(\zeta_{i,j} > y)dy
    \leq \int_0^{\infty} \min\left(1, \left(\frac{9}{y^2}+\frac{6}{y^4}\right)\frac{M^2}{T^2}\right) dy
    \leq 3\sqrt{\frac{M}{T}}+\frac{6M}{T}.
\end{align}
To conclude, we have
\begin{align}
    \eqref{UCB-negative-bound}= \sum_{t=0}^{T-1} \sum_{i,j} \mathbb E\left[ (r_{i,j} - \hat{r}_{i,j}(t)) x^*_{\epsilon,i,j} \right] \leq \left(3\sqrt{MT}+6M\right) \sum_{i}\lambda_i. 
\end{align}

\subsection{UCB learning}
Recall $\bar R_{i,j,s} = \frac{1}{s}\sum_{k=1}^sR_{i,j,k}$ is the empirical mean based on the first $s$ sampled rewards of a type $i$ job assigned to server $j.$ We study the upper bound on \eqref{UCB-negative-bound}. 
\begin{align*}
& \sum_{t=0}^{T-1} \sum_{i,j} \mathbb E\left[ (r_{i,j} - \hat{r}_{i,j}(t)) x^*_{\epsilon,i,j} \right]\\
\leq& \sum_{t=0}^{T-1} \sum_{i,j} \mathbb E\left[ (r_{i,j} - \hat{r}_{i,j}(t)) \mathbb I(r_{i,j} - \hat{r}_{i,j}(t) \geq 0)  x^*_{\epsilon,i,j} \right]\\
=& \sum_{t=0}^{T-1} \sum_{i,j} x^*_{\epsilon,i,j} \mathbb P\left( r_{i,j} - \hat{r}_{i,j}(t) \geq 0 \right).
\end{align*}
We bound the probability in the following. 
\begin{align}
\mathbb P\left( r_{i,j} - \hat{r}_{i,j}(t) \geq 0 \right) =& \mathbb P\left( r_{i,j} \geq \bar{r}_{i,j}(t-1) + \sqrt{\frac{\log T}{N_{i,j}(t-1)}} \right)\nonumber\\
\leq& \sum_{s = 1}^{C_{\lambda}T} \mathbb P\left( r_{i,j} \geq \bar R_{i,j,s} + \sqrt{\frac{\log T}{s}} \right)\label{ucb-tail-1}\\
\leq& \sum_{s = 1}^{C_{\lambda}T} \frac{2}{T^2}\label{ucb-tail-2}\\
\leq& \frac{2C_{\lambda}}{T}, \nonumber
\end{align}
where \eqref{ucb-tail-1} holds because the maximal arrival is bounded by $C_\lambda T$ and \eqref{ucb-tail-2} holds by Hoeffding inequality. It implies that
\begin{align*}
\sum_{t=0}^{T-1} \sum_{i,j} x^*_{\epsilon,i,j} \mathbb P\left( r_{i,j} - \hat{r}_{i,j}(t) \geq 0 \right) \leq \sum_{t=0}^{T-1} \frac{2C_{\lambda}}{T} \sum_{i,j} x^*_{\epsilon,i,j} \leq 2C_{\lambda}\sum_i \lambda_i. 
\end{align*}

\section{Proof of Lemma \ref{MAB-lemma}}
We study the upper bound on \eqref{UCB-bound}, which follows most of the steps in \cite{AudBub_09}. \begin{align*}
& \sum_{i,j}\sum_{t=0}^{T-1}\mathbb  E[(\hat r_{i,j}(t) - r_{i,j})x_{i,j}(t)] \\
=& \sum_{i,j}\sum_{t=0}^{T-1}\mathbb E[(\hat r_{i,j}(t) - r_{i,j}) \mathbb I(\hat r_{i,j}(t) - r_{i,j} \leq \zeta)x_{i,j}(t)] + \mathbb E[(\hat r_{i,j}(t) - r_{i,j}) \mathbb I(\hat r_{i,j}(t) - r_{i,j} > \zeta)x_{i,j}(t)] \\
\leq& T\zeta \sum_{i}\lambda_{i}  + \sum_{i,j}\sum_{t=0}^{T-1}\mathbb E[(\hat r_{i,j}(t) - r_{i,j}) \mathbb I(\hat r_{i,j}(t) - r_{i,j} > \zeta)x_{i,j}(t)] 
\end{align*}
where the last inequality holds because
\begin{align*}
\sum_{i,j}\sum_{t=0}^{T-1}\mathbb E[(\hat r_{i,j}(t) - r_{i,j}) \mathbb I(\hat r_{i,j}(t) - r_{i,j} 
\leq \zeta)x_{i,j}(t)] \\
\leq T\zeta \sum_{i,j}\mathbb E[x_{i,j}(t)] = T\zeta \sum_{i}\mathbb E[\Lambda_{i}(t)]=T\zeta \sum_{i} \lambda_i.
\end{align*}
The second term is bounded as follows
\begin{align*}
&\sum_{i,j}\sum_{t=0}^{T-1}\mathbb E[(\hat r_{i,j}(t) - r_{i,j}) \mathbb I(\hat r_{i,j}(t) - r_{i,j} > \zeta)x_{i,j}(t)] \\
\leq& \sum_{i}\mathbb E[\Lambda_{i}(t)]\sum_j \sum_{t=0}^{T-1}\mathbb E[(\hat r_{i,j}(t) - r_{i,j}) \mathbb I(\hat r_{i,j}(t) - r_{i,j} > \zeta)] \\
=& \sum_{i}\lambda_{i}\sum_{j} \sum_{t=0}^{T-1}\mathbb E[(\hat r_{i,j}(t) - r_{i,j}) \mathbb I(\hat r_{i,j}(t) - r_{i,j} > \zeta)].
\end{align*}

\subsection{MOSS learning}
Let $\zeta = \sqrt{\frac{2M}{T}},$ we focus on the term 
\begin{align*}
\sum_{t=0}^{T-1}\mathbb E[(\hat r_{i,j}(t) - r_{i,j}) \mathbb I(\hat r_{i,j}(t) - r_{i,j} > \zeta)] 
=& \int_{\zeta}^{\infty} \sum_{t=0}^{T-1} \mathbb P(\hat r_{i,j}(t) - r_{i,j} > y) dy \\
\leq& \int_{\zeta}^{\infty} \sum_{s=1}^{C_{\lambda}T} \mathbb P\left(\bar R_{i,j,s} + \sqrt{\frac{2}{s}\log \frac{T}{M s}}  - r_{i,j} > y\right) dy
\end{align*}
Let $s(y)=\ceil*{\frac{2\log\frac{Ty^2}{M}}{(1-c)^2y^2}}$ with $c=0.25.$ Therefore, we have
\begin{align*}
&\int_{\zeta}^{\infty} \sum_{s=1}^{C_{\lambda}T} \mathbb P\left(\bar R_{i,j,s} + \sqrt{\frac{2}{s}\log \frac{T}{M s}} - r_{i,j} > y\right) dy \\
\leq& \int_{\zeta}^{\infty} \frac{2\log\frac{Ty^2}{M}}{(1-c)^2y^2} + \sum_{s=s(y)}^{C_{\lambda} T} \mathbb P\left(\bar R_{i,j,s} + \sqrt{\frac{2}{s}\log \frac{T}{M s}} - r_{i,j} > y\right) dy
\end{align*}
The first term is bounded as follows:
\begin{align*}
\int_{\zeta}^{\infty} \frac{2\log\frac{Ty^2}{M}}{(1-c)^2y^2} dy = \frac{2}{(1-c)^2} \frac{\log (\frac{T\zeta^2}{M})+2}{\zeta}
\end{align*}
The second term is bounded as follows:
\begin{align*}
&\int_{\zeta}^{\infty} \sum_{s=s(y)}^{C_{\lambda} T} \mathbb P\left(\bar R_{i,j,s} + \sqrt{\frac{2}{s}\log \frac{T}{M s}} - r_{i,j} > y\right) dy \\
\leq& \int_{\zeta}^{\infty} \sum_{s=s(y)}^{C_{\lambda} T} \mathbb P(\bar R_{i,j,s} - r_{i,j} > cy) dy\\
=& \int_{\zeta}^{\frac{1}{c}} \sum_{s=s(y)}^{C_{\lambda} T} \mathbb P(\bar R_{i,j,s} - r_{i,j} > cy) dy\\
\leq& \int_{\zeta}^{\frac{1}{c}} \sum_{s=s(y)}^{C_{\lambda} T} e^{-2sc^2y^2} dy \\
\leq& \int_{\zeta}^{\frac{1}{c}} \frac{e^{-\frac{4c^2}{(1-c)^2}\log \frac{T\zeta^2}{M}}} {1-e^{-2c^2y^2}}dy
\end{align*}
where the first inequality holds because the value of $s(y)=\ceil*{\frac{2\log\frac{Ty^2}{M}}{(1-c)^2y^2}}$ and $y \geq \zeta=\sqrt{\frac{2M}{T}};$ the second inequality holds because the assumption that the reward is in $[0, 1];$ the third inequality holds by Hoeffding inequality. 
Now the last term is bounded as follows
\begin{align*}
\int_{\zeta}^{\frac{1}{c}} \frac{1} {1-e^{-2c^2y^2}}dy =& \int_{\zeta}^{\frac{1}{2c}} \frac{1} {2c^2y^2 - 2c^4y^4}dy + \int_{\frac{1}{2c}}^{\frac{1}{c}} \frac{1} {1-e^{-2c^2y^2}}dy\\
\leq& \int_{\zeta}^{\frac{1}{2c}} \frac{2} {3c^2y^2}dy + \frac{1} {2c(1-e^{-0.5})}\\
=& \frac{2}{3c^2\zeta} - \frac{4}{3 c} + \frac{1} {2c(1-e^{-0.5})} \\
\leq& \frac{2}{3c^2\zeta}
\end{align*}
Recall $c=0.25$ and we have
\begin{align*}
&\sum_{t=0}^{T-1}\mathbb E[(\hat r_{i,j}(t) - r_{i,j}) \mathbb I(\hat r_{i,j}(t) - r_{i,j} > \zeta)] \\
\leq&  \frac{2}{(1-c)^2} \frac{\log (\frac{T\zeta^2}{M})+2}{\zeta} + \frac{2}{3c^2\zeta} e^{-\frac{4c^2}{(1-c)^2}\log \frac{T\zeta^2}{M}}\\
\leq& 19\sqrt{\frac{T}{M}}
\end{align*}
which implies
\begin{align*}
\sum_{i}\lambda_i\sum_{j}\sum_{t=0}^{T-1}\mathbb E[(\hat r_{i,j}(t) - r_{i,j}) \mathbb I(\hat r_{i,j}(t) - r_{i,j} > \zeta)]\mathbb \leq 19\sqrt{TM} \sum_{i} \lambda_i
\end{align*}
Substitute $\zeta = \frac{2M}{T},$ we have
\begin{align}
\eqref{UCB-bound} \leq 2 \sqrt{TM} \sum_{i} \lambda_i + 19 \sqrt{TM} \sum_{i} \lambda_i = 21\sqrt{TM} \sum_{i} \lambda_i
\end{align}

\subsection{UCB learning}
Let $\zeta = \sqrt{\frac{3M\log T}{T}},$ we focus on the term 
\begin{align*}
\sum_{t=0}^{T-1}\mathbb E[(\hat r_{i,j}(t) - r_{i,j}) \mathbb I(\hat r_{i,j}(t) - r_{i,j} > \zeta)] 
=& \int_{\zeta}^{\infty} \sum_{t=0}^{T-1} \mathbb P(\hat r_{i,j}(t) - r_{i,j} > y) dy \\
\leq& \int_{\zeta}^{\infty} \sum_{s=1}^{C_{\lambda}T} \mathbb P\left(\bar R_{i,j,s} + \sqrt{\frac{\log T}{s}} - r_{i,j} > y\right) dy
\end{align*}

Let $s(y) = \ceil*{\frac{\log T}{(1-c)^2y^2}}$ and $c = 1-\frac{1}{\sqrt{3}},$ we have 
\begin{align*}
&\int_{\zeta}^{\infty} \sum_{s=1}^{C_{\lambda}T} \mathbb P\left(\bar R_{i,j,s} + \sqrt{\frac{\log T}{s}} - r_{i,j} > y\right) dy \\
\leq&  \int_{\zeta}^{\infty} s(y) dy + \int_{\zeta}^{\infty} \sum_{s=\ceil{s(y)}}^{C_{\lambda}T}\mathbb P\left(\bar R_{i,j,s} + \sqrt{\frac{\log T}{s}} - r_{i,j} > y\right) dy \\
\leq& \int_{\zeta}^{\infty} \frac{3\log T}{y^2} dy + \int_{\zeta}^{\infty} \sum_{s=\ceil{s(y)}}^{C_{\lambda}T}\mathbb P\left(\bar R_{i,j,s} - r_{i,j} > cy\right) dy.
\end{align*}

By following the steps as in MOSS learning, we have 
\begin{align*}
\int_{\zeta}^{\infty} \sum_{s=\ceil{s(y)}}^{C_{\lambda}T}\mathbb P\left(\bar R_{i,j,s} - r_{i,j} > cy\right) dy \leq& \frac{4}{T\zeta}.
\end{align*}
and 
\begin{align*}
\int_{\zeta}^{+\infty} \frac{3\log T}{y^2} dy = \frac{3\log T}{\zeta}.
\end{align*}

Substitute $\zeta = \sqrt{\frac{3M\log T}{T}},$ we have that
\begin{align}
\eqref{UCB-bound} \leq 4\sqrt{MT\log T} \sum_{i} \lambda_i.
\end{align}

\section{Proof of Lemma \ref{negative drift}} \label{app: drift-lemma}
\driflemma*
\begin{proof}
According to the conditional expected drift in \eqref{drif ucb} (given $\mathbf H(t)=\mathbf h$):
\begin{align}
& \mathbb E[L(t+1)-L(t)|\mathbf H(t)=\mathbf h] \nonumber\\
\leq& \sum_j\sum_{k}Q_j^{(k)}\left(\sum_{i}w_{i,j}^{(k)} x_{\epsilon,i,j} - \rho_j^{(k)}+\epsilon\right) -  V\mathbb E\left[\sum_{i,j} \hat r_{i,j}(t)x_{\epsilon,i,j}-\sum_{i,j} \hat r_{i,j}(t)x_{i,j}(t)|\mathbf H(t)=\mathbf h\right] + B \nonumber \\
\leq& \sum_j\sum_{k}Q_j^{(k)}\left(\sum_{i}w_{i,j}^{(k)} x_{\epsilon,i,j} - \rho_j^{(k)}+\epsilon\right) + V\mathbb E\left[\sum_{i,j} \hat r_{i,j}(t)x_{i,j}(t)|\mathbf H(t)=\mathbf h\right] + B \nonumber \\
\leq& \sum_j\sum_{k}Q_j^{(k)}\left(\sum_{i}w_{i,j}^{(k)} x_{\epsilon,i,j} - \rho_j^{(k)}+\epsilon\right) + V \sum_i \lambda_i + B \nonumber \\
\leq& -\left(\frac{3\delta}{4} - \epsilon\right)\sum_j\sum_{k}Q_j^{(k)} + V \sum_i \lambda_i + B \label{cond:drift-exact}
\end{align}
where the second inequality holds because the utility function is always positive $\mathbf x_{\epsilon} \in \mathcal X_{\epsilon};$ the third inequality holds because rewards are $[0,1];$ the last inequality holds because of Lemma \ref{lemma: tightness} below.
\end{proof}

\begin{lemma}\label{lemma: tightness}
Suppose there exists $\mathbf x \in \mathcal X$ satisfying $\sum_{i}w_{i,j}^{(k)} x_{i,j} - \rho_j^{(k)} \leq -\delta, \forall j, k$ as in Assumption \ref{assumption:slater} and $\delta \geq 4 \epsilon,$ there always exists $\mathbf x_\epsilon$ such that $$\sum_{i}w_{i,j}^{(k)} x_{\epsilon,i,j} - \rho_j^{(k)}+\epsilon \leq -\frac{3\delta}{4}, \forall k.$$
\begin{proof}
Under Assumption \ref{assumption:slater}, there exists $\mathbf x \in \mathcal X$ satisfying $\sum_{i}w_{i,j}^{(k)} x_{i,j} - \rho_j^{(k)} \leq -\delta, \forall j, k,$ we can always find $\mathbf x_\epsilon$ with $\delta \geq 4 \epsilon$ such that
$$\sum_{i}w_{i,j}^{(k)} x_{\epsilon,i,j} - \rho_j^{(k)}+\epsilon \leq \sum_{i}w_{i,j}^{(k)} x_{i,j} - \rho_j^{(k)}+\epsilon  \leq -\delta + \epsilon \leq -\frac{3\delta}{4}, ~\forall j, k.$$
\end{proof}
\end{lemma}

\end{document}